\documentclass{article}

\PassOptionsToPackage{numbers}{natbib}
\usepackage[preprint]{neurips_2026}

\usepackage[utf8]{inputenc} 
\usepackage[T1]{fontenc}    
\usepackage{hyperref}       
\usepackage{url}            
\usepackage{booktabs}       
\usepackage{amsfonts}       
\usepackage{nicefrac}       
\usepackage{microtype}      
\usepackage{xcolor}         

\usepackage{hyperref}       %
\usepackage{url}            %
\usepackage{booktabs}       %
\usepackage{amsfonts}       %
\usepackage{nicefrac}       %
\usepackage{microtype}      %
\usepackage{xcolor}         %
\usepackage{amsmath,amsfonts,bm}
\usepackage{amssymb}
\usepackage{graphicx}
\usepackage{xspace}
\usepackage{mathtools}
\usepackage{amsthm}
\usepackage{amssymb}
\usepackage{mathabx}

\usepackage{tikz}
\usetikzlibrary{trees}

\usepackage{algorithm, algorithmic}
\usepackage{wrapfig}

\usepackage{tikz}
\usetikzlibrary{trees}

\newtheorem{theorem}{Theorem}[section]
\newtheorem{corollary}[theorem]{Corollary}
\newtheorem{lemma}[theorem]{Lemma}
\newtheorem{proposition}[theorem]{Proposition}

\newtheorem{example}{Example}[section]

\newtheorem{remark}{Remark}[section]
\newtheorem{assumption}{Assumption}[section]

\newcommand{\argmin}{\mathop{\mathrm{argmin}}}
\newcommand{\argmax}{\mathop{\mathrm{argmax}}}

\usepackage[ruled,algo2e]{algorithm2e}
\SetKwComment{Comment}{$//$ }{}
\usepackage{subcaption}

\title{Future Information-Directed Sampling for Bayesian Nonstationary Bandits}

\author{%
  Yichen Song \\
  Boston University\\
  \texttt{ycs@bu.edu} \\
  \And
  Alessio Russo \\
  Boston University \\
  \texttt{arusso2@bu.edu} \\
  \AND
  Aldo Pacchiano \\
  Boston University \\
  Broad Institute of MIT and Harvard \\
  \texttt{pacchian@bu.edu}\\
}

\begin{document}

\maketitle

\begin{abstract}
    Exploration--exploitation is a central trade-off in bandit learning. While classical algorithms such as upper confidence bound methods and Thompson Sampling effectively balance this trade-off in stationary environments, their exploration strategies mainly reduce uncertainty about the current optimal arm, which can be insufficient in nonstationary settings where future optimal arms may differ substantially from current ones. In this paper, we propose Future Information-Directed Sampling (FIDS), a new algorithm for Bayesian nonstationary bandits that explicitly explores to gather information about future optimal arms. We show that FIDS achieves regret comparable to Thompson Sampling up to a small constant factor, while being able to exploit predictive information structures that conventional exploration objectives fail to capture. To address the practical difficulty of posterior inference, we further propose a supervised-learning-based approximation framework that learns the FIDS policy from offline data, and demonstrate its effectiveness on synthetic benchmarks.
\end{abstract}

\section{Introduction}\label{sec:intro}
In this paper, we study multi-armed bandits (MABs) in nonstationary environments. MAB problems have been extensively studied over the past decades as a fundamental framework for sequential decision making \cite{lai1985Asymptotically,lai1987Adaptive,lattimore2020bandit}, primarily under the assumption that the underlying environment is stationary. However, this assumption is often unrealistic in real-world applications such as recommendation systems \cite{clerici2023linear,pandey2007bandits} and online auctions \cite{kleinberg2003value}, where the environment may evolve over time. Motivated by these applications, we consider the nonstationary MAB setting \cite{garivier2011upper,cheung2019learning}, in which the reward distributions can change over time. We further adopt a Bayesian perspective, where the sequence of environments is generated according to a known prior distribution, as considered in \cite{mellor2013thompson,liu2023nonstationary,min2023information}. We refer to this setting as Bayesian nonstationary bandits.

Thompson Sampling (TS) \cite{thompson1933likelihood}, one of the most widely used algorithms for Bayesian bandit problems, has been analyzed in both stationary \cite{russo2018tutorial} and nonstationary settings \cite{min2023information}. Despite the differences between these settings, the analyses rely on a common key ingredient: for every round \(t\), the information ratio induced by TS can be shown to be upper bounded by a suitable quantity. The information ratio is defined as the ratio between the squared expected instantaneous regret and the mutual information gained from the next observation about the optimal arm at round (t). A formal definition is provided in Proposition~\ref{prop:generilized TS property}. This quantity captures the fundamental exploration--exploitation trade-off in bandit learning. Specifically, a small information ratio implies that either the algorithm incurs only a small instantaneous regret (exploitation), or it acquires a substantial amount of information about the optimal arm (exploration).

While this property is highly desirable in stationary environments, it is less clear whether the exploration--exploitation trade-off characterized by the information ratio remains appropriate in nonstationary settings. In particular, the exploration component rewards information acquisition about the optimal arm at the current round. However, in a nonstationary environment, the identity of the optimal arm may change over time. Consequently, information that is useful for identifying the current optimal arm may quickly become stale and contribute little to future decision making.

Motivated by this observation, we propose \emph{Future Information-Directed Sampling} (FIDS), a new algorithm that explicitly prioritizes gathering information about future optimal arms. At a high level, FIDS selects actions by solving an optimization problem that balances instantaneous regret against the mutual information between future optimal arms and the next observation. In this way, FIDS directs exploration toward information that is expected to remain valuable in the future, rather than solely toward identifying the currently optimal arm.

We establish theoretical guarantees for FIDS and demonstrate its empirical effectiveness through experiments. Our main contributions are summarized below.
\begin{enumerate}
    \item We prove in Theorem~\ref{thm:regret} that FIDS achieves the same regret upper bound as Thompson Sampling in Bayesian nonstationary environments, up to a small constant factor.
    
    \item Through several analytic examples, we demonstrate that FIDS can exploit predictive information structures that conventional exploration objectives fail to capture, leading to significantly improved performance over baseline methods.
    
    \item In practice, the prior distribution may be unknown and posterior inference may be computationally intractable. To sidestep these difficulties, we propose a supervised-learning framework that learns the FIDS policy with a truncated-window approximation from offline data, inspired by DPT \cite{lee2023supervised}, which similarly trains a Thompson Sampling policy via supervision. Empirically, the learned FIDS policy outperforms DPT and other baselines by effectively exploiting the information structure of the environment.

    \item In Section \ref{sec:full fids}, we present a new training method that removes the truncated-window approximation and thus recovers the full FIDS policy at optimality. We leave the empirical evaluation of this new training method for future work.
\end{enumerate}

\section{Preliminaries and Problem Formulation}\label{sec:problem formulation}
In this section we introduce the problem formulation, and provide background knowledge on Thompson Sampling and Information Directed Sampling.

We consider a finite MAB problem with horizon \(n\) and an action set \(\mathcal{A}\) with  \(|\mathcal{A}| = K\) arms. 
We model the non-stationary environment by a random variable \(\nu\) drawn from a prior distribution known to the learner. Here \(\nu = (\nu_{t,a})_{t \in [n], a \in \mathcal{A}}\) where each \(\nu_{t,a}\) specifies the reward distribution of arm \(a\) at round \(t\). Consider \(R_{t,a}\) to be the reward one will see if they play arm \(a\) at round \(t\). We assume that conditional on \(\nu_{t,a}\), the reward \(R_{t,a}\) is sampled independently as \(R_{t,a} \sim \nu_{t,a}\).  As a special case, if for every \(a \in \mathcal{A}\), we have \(\nu_{t,a} = \nu_{s,a}\) for all \(t,s \in [n]\) almost surely, then the model reduces to the standard stationary Bayesian bandit setting. 

In each round \(t\), the learner interacts with the environment by selecting an action \(A_t\) according to a policy \(\pi_t\), which maps the history \(\mathcal{F}_{t-1} = (A_1,R_1,\dots,A_{t-1},R_{t-1})\) to a distribution over \(\mathcal{A}\), and observes reward \(R_t \triangleq R_{t,A_t}\). 
Let \(\pi = (\pi_t)_{t=1}^n\). We evaluate \(\pi\) using the following Bayesian regret:
\begin{align}
\text{Regret}(\pi,n) 
= \mathbb{E}\left[\sum_{t=1}^n \big( R_{t,A_t^\star} - R_{t,A_t} \big)\right],
\end{align}
where \(A_t^\star \in \arg\max_{a \in \mathcal{A}} \mu(\nu_{t,a})\). And \(\mu(\nu_{t,a})\) denotes the mean of the distribution \(\nu_{t,a}\). We also make the following assumption on the reward distributions.

\begin{assumption}\label{subgaussian assumption}
    There exists a constant \(\sigma \in \mathbb{R}\), such that for all \(a \in \mathcal{A}\), conditioned on \(\mathcal{F}_t\), \(R_{t,a} - \mathbb{E}\left[R_{t,a} \mid \mathcal{F}_t\right]\) is \(\sigma\) sub-Gaussian. And we assume \(\sigma\) is known to the learner.
\end{assumption}



\noindent{\bf Notation.}
We introduce the following shorthand notations that will be used throughout the analysis. For \(i,j \in [n]\) with \(i \leq j\), let \(A^\star_{i:j} \triangleq \left[A^\star_i,\dots,A^\star_j\right]\).
Let \(\mathbb{P}_t(\cdot) \triangleq \mathbb{P}(\cdot \mid \mathcal{F}_{t-1})\) denote the posterior measure given the history up to round \(t-1\). For random variables \(X\) and \(Y\), we denote \(\mathbb{E}_t[X]\), \(H_t(X)\), \(H_t(X \mid Y)\), and \(I_t(X;Y)\) to be the expectation, entropy, conditional entropy, and mutual information under the posterior \(\mathbb{P}_t\), respectively. 
Formal definitions of these information-theoretic quantities can be found in Appendix~\ref{sec:information preliminaries}. Note that \(\mathbb{E}_t[X]\), \(H_t(X)\), \(H_t(X \mid Y)\), and \(I_t(X;Y)\) are all random variables, since they are functions of the history \(\mathcal{F}_{t-1}\) which is a random variable.

\subsection{Preliminaries about TS and IDS}\label{sec:preliminaries about TS}
We now recall a key result from the information-theoretic analysis of TS by \cite{russo2016information}, which underpins both the design of FIDS and its regret analysis.


\begin{proposition}{(Corollary 1 in \cite{russo2016information})}\label{prop:generilized TS property}
Denote the Thompson sampling policy by \(\pi^{\text{TS}} = \left(\pi_t^{\text{TS}}\right)_{t=1}^n\). Suppose that conditioned on \(\mathcal{F}_t\), \(R_{t,a} - \mathbb{E}\left[R_{t,a} \mid \mathcal{F}_t\right]\) is \(\sigma\) sub-Gaussian, where \(\sigma\) is a constant number in \(\mathbb{R}\). Then,
    \begin{align}\label{eq:info ratio def}
    \frac{\sum_{a \in \mathcal{A}} \pi_t^{\text{TS}}(a) \mathbb{E}_t\left[R_{t,A_t^\star} - R_{t,a}\right]}{\sqrt{\sum_{a \in \mathcal{A}} \pi_t^{\text{TS}}(a) I_t\left(A_t^\star; R_{t,a}\right)}} \leq \sigma\sqrt{2K},\quad \forall t\in [n], \text{ a.s.}
\end{align} 
\end{proposition}
The left-hand side of Equation~\eqref{eq:info ratio def} is known as the \emph{information ratio}. Using this bound, \cite{russo2016information} show that the Bayesian regret of TS is at most $\sigma\sqrt{2nKH(A^\star)}$, where $A^\star \triangleq A_1^\star = \cdots = A_n^\star$ since the environment is stationary.


In addition, IDS \cite{russo2014learning} selects a policy \(\pi_t^{\mathrm{IDS}}\) that explicitly minimizes the information ratio at each round \(t\), i.e.,
\begin{align}\label{eq:def of ids}
    \pi_t^{\text{IDS}} \in \argmin_{\pi_t \in \Delta\left(\mathcal{A}\right)} \frac{\sum_{a \in \mathcal{A}} \pi_t(a) \mathbb{E}_t\left[R_{t,A_t^\star} - R_{t,a}\right]}{\sqrt{\sum_{a \in \mathcal{A}} \pi_t(a) I_t\left(A_t^\star; R_{t,a}\right)}},
\end{align}
IDS enjoy the same worst-case regret bound as TS, but can achieve significantly improved performance in certain scenarios (see \cite{russo2014learning} for more details).

Using the same proof technique, \cite{min2023information} analyze TS in the nonstationary setting of Section~\ref{sec:problem formulation} and obtain a regret bound of $O(\sigma\sqrt{2nKH([A_1^\star,\dots,A_n^\star])})$. This subsumes the stationary result: when the environment is stationary, $A_1^\star = \cdots = A_n^\star = A^\star)$ almost surely, so $H([A_1^\star,\dots,A_n^\star]) = H(A^\star)$.


\section{FIDS Algorithm}\label{sec:algorithm}

In this section, we develop FIDS by addressing a key limitation of TS and IDS in nonstationary environments. In stationary settings, both algorithms achieve low regret by bounding the information ratio
\[\frac{\sum_{a \in \mathcal{A}} \pi_t(a),\mathbb{E}t[R_{t,A_t^\star} - R_{t,a}]}{\sqrt{\sum_{a \in \mathcal{A}} \pi_t(a),I_t(A_t^\star; R_{t,a})}},\]
which controls the trade-off between instantaneous regret and information gained about the optimal arm $A_t^\star$. This works because $A_t^\star$ is fixed, so reducing uncertainty about it directly improves future decisions. In nonstationary environments, however, $A_t^\star$ may vary with $t$, so information gained about the current optimum need not reduce uncertainty about future optimal arms. FIDS addresses this by replacing the per-round information term with one that targets the sequence of optimal arms.

To address the exploration-exploration problem in a meaningful way, we introduce \emph{Future Information-Directed Sampling} (FIDS), which balances instantaneous regret against information gained about future optimal arms. Under Assumption~\ref{subgaussian assumption}, FIDS is defined by the policy sequence $\pi^{\text{FIDS}} = (\pi_t^{\text{FIDS}})_{t=1}^n$, where $\pi_t^{\text{FIDS}}$ is any policy that solves the following optimization problem
\begin{align}\label{eq:definition of FIDS}
    \pi^{\text{FIDS}}_t \in \argmin_{\pi_t \in \Delta(\mathcal{A})} \sum_{a \in \mathcal{A}}\pi_t(a)\mathbb{E}_t[R_{t,A^\star_t}-R_{t,a}] - \sigma\sqrt{2K\cdot \sum_{a \in \mathcal{A}}\pi_t(a) I_t\left(A^\star_{t+1:n};R_{t,a}\right)}
\end{align}

Note that solving Equation~\eqref{eq:definition of FIDS} is equivalent to solving the following Equation~\eqref{eq:equivalent definition of FIDS} 
\begin{align}\label{eq:equivalent definition of FIDS}
    \pi^{\text{FIDS}}_t \in \argmax_{\pi_t \in \Delta(\mathcal{A})} \sum_{a \in \mathcal{A}}\pi_t(a)\mathbb{E}_t[R_{t,a}] + \sigma\sqrt{2K\cdot \sum_{a \in \mathcal{A}}\pi_t(a) I_t(A^\star_{t+1:,n};R_{t,a})},
\end{align}
as \(\sum_{a \in \mathcal{A}}\pi_t(a)\mathbb{E}_t\left[R_{t,A^\star_t}\right] = \mathbb{E}_t\left[R_{t,A^\star_t}\right]\) is a constant with respect to \(\pi_t\).

A useful property of Equation~\eqref{eq:equivalent definition of FIDS} is that the optimization in Equation~\eqref{eq:equivalent definition of FIDS} always admits an optimal policy supported on at most two arms, as shown in the following lemma (proof deferred to the appendix in section ~\ref{app:proof_lemma_2supp}).
\begin{lemma}\label{lemma:2_supp}
    Given \(\mathbb{E}_t\left[R_{t,a}\right]\) and \(I_t\left(\left[A^\star_{t+1},\dots,A^\star_n\right];R_{t,a}\right)\), Equation~\eqref{eq:equivalent definition of FIDS} is maximized at some \(\pi^\star\) which has at most 2 non-zero elements.
\end{lemma}
Intuitively, the objective in \eqref{eq:equivalent definition of FIDS} is concave in $\pi_t$, so at any optimum every arm in the support must have the same marginal contribution. This equal-gradient condition forces a linear relationship between the reward and information values of the supported arms, and in the proof we show that any optimal mixture can be replicated by a two-point mixture of the arms  in the support.

Before proceeding with the regret bound of FIDS, we note a subtlety in applying IDS to the nonstationary setting.
\begin{remark}
In nonstationary environments, $\sum_{a \in \mathcal{A}} \pi_t(a), I_t(A_t^\star; R_{t,a})$ can equal zero: even if the current optimal arm $A_t^\star$ is known with certainty, the future optimal arms $A_{t+1}^\star, \dots, A_n^\star$ may remain unknown, so the problem is still nontrivial. Because the standard ratio form of IDS is undefined when this information term vanishes, we adopt the following additive variant for nonstationary environments:
\begin{align}\label{eq:definition of NIDS}
    \pi^{\text{IDS(ns)}}_t \in \argmax_{\pi_t \in \Delta(\mathcal{A})} \sum_{a \in \mathcal{A}}\pi_t(a)\mathbb{E}_t[R_{t,a}] + \sigma\sqrt{2K\cdot \sum_{a \in \mathcal{A}}\pi_t(a) I_t(A^\star_{t};R_{t,a})}
\end{align}
All experiments involving IDS use this formulation.
\end{remark}

We now show that FIDS achieves a regret bound comparable to that of TS in nonstationary environments.
\begin{theorem}\label{thm:regret}
Under Assumption~\ref{subgaussian assumption},
\begin{align}
\mathrm{Regret}(\pi^{\mathrm{FIDS}}, n) \leq 2\sigma\sqrt{2K \cdot n \cdot H([A_1^\star, \dots, A_n^\star])}
\end{align}
\end{theorem}
This bound matches the TS regret bound of \cite{min2023information} up to a constant factor of 2. Despite this worst-case similarity, FIDS can exploit additional structure in the environment; in Section~\ref{sec:examples}, we show that it achieves strictly better performance in several settings. Here is the proof of the theorem.
\begin{proof}
Here we outline the proof with several supporting lemmas. The proof of lemmas can be found in Section \ref{sec:supporting lemmas}.
For ease of the notation, in this section, we use \(\Delta_t\left(\pi_t\right) \coloneqq \sum_{a \in \mathcal{A}}\pi_t(a)\mathbb{E}_t[R_{t,A^\star_t}-R_{t,a}]\), \(\tilde{g}_t\left(\pi_t\right) \coloneqq \sum_{a \in \mathcal{A}}\pi_t(a)I_t([A^\star_{t+1},\dots, A^\star_n];R_{t,a})\) and \(g_t\left(\pi_t\right) \coloneqq \sum_{a \in \mathcal{A}}\pi_t(a)I_t(A^\star_t;R_{t,a})\). When an action \(A_t\) is sampled from \(\pi_t\), we slightly abuse the notation by letting \(\Delta_t\left(A_t\right) \coloneqq\mathbb{E}_t\left[R_{t,A_t^\star} - R_{t,A_t}\right], \tilde{g}_t\left(A_t\right) \coloneqq I_t\left(\left[A_{t+1}^\star,\dots,A_n^\star\right]; \left(A_t, R_{t,A_t}\right)\right)\) and \(g_t\left(A_t\right) \coloneqq I_t\left(A_{t}^\star; \left(A_t, R_{t,A_t}\right)\right)\), as Lemma \ref{lemma:transformation} shows that \(\Delta_t\left(A_t\right) = \Delta_t\left(\pi_t\right)\), \(\tilde{g}_t\left(A_t\right) =\tilde{g}_t\left(\pi_t\right)\), and \(g_t\left(A_t\right) = g_t\left(\pi_t\right)\).

The first part of the proof connects \(\pi^{\text{FIDS}}\) with \(\pi^{\text{TS}}\). By the Tower Rule, \(\text{Regret}\left(\pi^{\text{FIDS}},n\right) = \mathbb{E}\left[\sum_{t=1}^n \Delta_t\left(\pi^{\text{FIDS}}_t\right)\right]\), then by adding and subtracting \(\sigma\sqrt{2K\cdot \tilde{g}_t\left(\pi^{\text{FIDS}}_t\right)}\), we get
\begin{align*}
    \text{Regret}\left(\pi^{\text{FIDS}},n\right)
    &=\mathbb{E}\left[\sum_{t=1}^n\Delta_t\left(\pi^{\text{FIDS}}_t\right) - \sigma\sqrt{2K\cdot \tilde{g}_t\left(\pi^{\text{FIDS}}_t\right)} + \sigma\sqrt{2K\cdot \tilde{g}_t\left(\pi^{\text{FIDS}}_t\right)}\right]\\
    &\stackrel{(a)}{\leq} \mathbb{E}\left[\sum_{t=1}^n\Delta_t\left(\pi^{\text{TS}}_t\right) - \sigma\sqrt{2K\cdot \tilde{g}_t\left(\pi^{\text{TS}}_t\right)} + \sigma\sqrt{2K\cdot \tilde{g}_t\left(\pi^{\text{FIDS}}_t\right)}\right]\\
    &\stackrel{(b)}{=} \mathbb{E}\left[\sum_{t=1}^n\Delta_t\left(\pi^{\text{TS}}_t\right) - \sigma\sqrt{2K\cdot g_t\left(\pi^{\text{TS}}_t\right)} - \sigma\sqrt{2K\cdot \tilde{g}_t\left(\pi^{\text{TS}}_t\right)}\right.\\ &\left. \qquad+ \sigma\sqrt{2K\cdot \tilde{g}_t\left(\pi^{\text{FIDS}}_t\right)} + \sigma\sqrt{2K\cdot g_t\left(\pi^{\text{TS}}_t\right)}\right]\\
    &\stackrel{(c)}{\leq} \mathbb{E}\left[\sum_{t=1}^n \sigma\sqrt{2K\cdot g_t\left(\pi^{\text{TS}}_t\right)} - \sigma\sqrt{2K\cdot \tilde{g}_t\left(\pi^{\text{TS}}_t\right)} + \sigma\sqrt{2K\cdot \tilde{g}_t\left(\pi^{\text{FIDS}}_t\right)}\right]
\end{align*}
where (a) uses the fact that \(\pi^{\text{FIDS}}_t\) is the minimizer of Equation \eqref{eq:definition of FIDS}; (b) adds and subtracts \(\sigma\sqrt{2K\cdot g_t\left(\pi^{\text{TS}}_t\right)}\); (c) holds because of Proposition \ref{prop:generilized TS property}. Then by Lemma \ref{lemma:bound term 2}, 
\begin{align}
    \mathbb{E}\left[\sum_{t=1}^n \sigma\sqrt{2K\cdot \tilde{g}_t\left(\pi^{\text{FIDS}}_t\right)}\right] \leq \sigma\sqrt{2n\cdot K \cdot H(A^\star_{1:n})}\label{eq:middle step1}
\end{align}
 by Lemma \ref{lemma:sum term1}, 
\begin{align}
    \mathbb{E}\left[\sum_{t=1}^n \sigma\sqrt{2K\cdot g_t\left(\pi^{\text{TS}}_t\right)} - \sigma\sqrt{2K\cdot \tilde{g}_t\left(\pi^{\text{TS}}_t\right)}\right] \leq \sigma\sqrt{2n\cdot K \cdot H(A^\star_{1:n})}\label{eq:middle step2}
\end{align}

Combining Equation \eqref{eq:middle step1} and \eqref{eq:middle step2} proves the theorem.
\end{proof}

In the next subsection, we illustrate settings in which TS and IDS fail to exploit the information structure of the environment, while FIDS does not.


\subsection{Illustrative Examples: When Does Information Structure Matter?}\label{sec:examples}
We present two examples that illustrate the role of information structure. The first is an i.i.d.\ environment in which the optimal arm is redrawn independently each round, so observations carry no information about future optimal arms. The second features a change point where information about the future environment is encoded in a currently suboptimal arm; we complement the analytical results for this last example with numerical experiments.

\begin{example}[Memoryless environment]\label{ex:i.i.d. examples}
At each round \(t\), the environment \(\nu_t\) is sampled i.i.d. with \(K=2\), where \((\mu_{t,1}, \mu_{t,2})=(0.4,0.2)\) with probability \(1/2\) and \((\mu_{t,1}, \mu_{t,2})=(0.3,0.4)\) with probability \(1/2\). Conditional on \(\nu_t\), rewards satisfy \(R_{t,a}\sim\mathrm{Bernoulli}(\mu_{t,a})\) for \(a \in \{1,2\}\).
\end{example}

In this setting, exploration is futile: no observation at round $t$ can reduce uncertainty about future environments. As shown in Proposition~\ref{prop:iid examples}, FIDS recognizes this and always selects arm~1, the arm with the highest expected instantaneous reward, achieving cumulative reward $0.35n$. In contrast, TS and IDS allocate probability to unnecessary exploration, achieving only $0.325n$ and $0.3294n$ respectively.


Example \ref{ex:i.i.d. examples} serves as an environment where no future information can be gathered.
We next consider a setting where information about future optimal arms is encoded in a currently suboptimal arm, and FIDS is able to exploit this cross-environment structure.


\begin{example}[Magic arm environment]\label{ex:informative ex}
We consider a non-stationary 6-arm bandit with horizon \(n=100\) and a single change point at \(t=50\). 
The environment is \textbf{Env 1} for \(t \le 50\) and \textbf{Env 2} for \(t > 50\), and is stationary within each segment.

\begin{itemize}
    \item[{\texttt{Env 1}}.]
    Let \(Z_1 \sim \mathrm{Unif}(\{1,\dots,5\})\). 
For \(a \in \{1,\dots,5\}\), \(\mu^1_a = 0.9\) if \(a=Z_1\) and \(\mu^1_a = 0.4\) otherwise, with \(R_{t,a} \sim \mathcal{N}(\mu^1_a, 0.1^2)\). 
For \(a=6\), \(\mu^1_6 = Z_2/7\) and \(R_{t,6} \sim \mathcal{N}(\mu^1_6, 0.1^2)\).

    \item[{\texttt{Env 2}}.]
    Let \(Z_2 \sim \mathrm{Unif}(\{1,\dots,6\})\), independent of \(Z_1\). 
For all \(a \in \mathcal{A}\), \(\mu^2_a = 3.0\) if \(a=Z_2\) and \(\mu^2_a = 0.5\) otherwise, with \(R_{t,a} \sim \mathcal{N}(\mu^2_a, 0.5^2)\).
\end{itemize}
\end{example}

 Note that the information about \texttt{Env 2} is encoded in the mean reward of arm \(6\) in \texttt{Env 1}. In particular, if \(\mu_{1,6}\) is known exactly, then the means of all arms in \texttt{Env 2} are immediately determined. FIDS is able to exploit this structure through its objective of gathering information about future optimal arms. In contrast, TS and IDS are designed to gather information only about the currently optimal arm. Since arm \(6\) is never optimal in \texttt{Env 1}, these methods fail to actively explore it and therefore cannot effectively exploit the information it provides about \texttt{Env 2}.

\begin{figure}[t]
    \centering
    
    \begin{subfigure}[t]{0.32\linewidth}
        \centering
        \includegraphics[width=\linewidth]{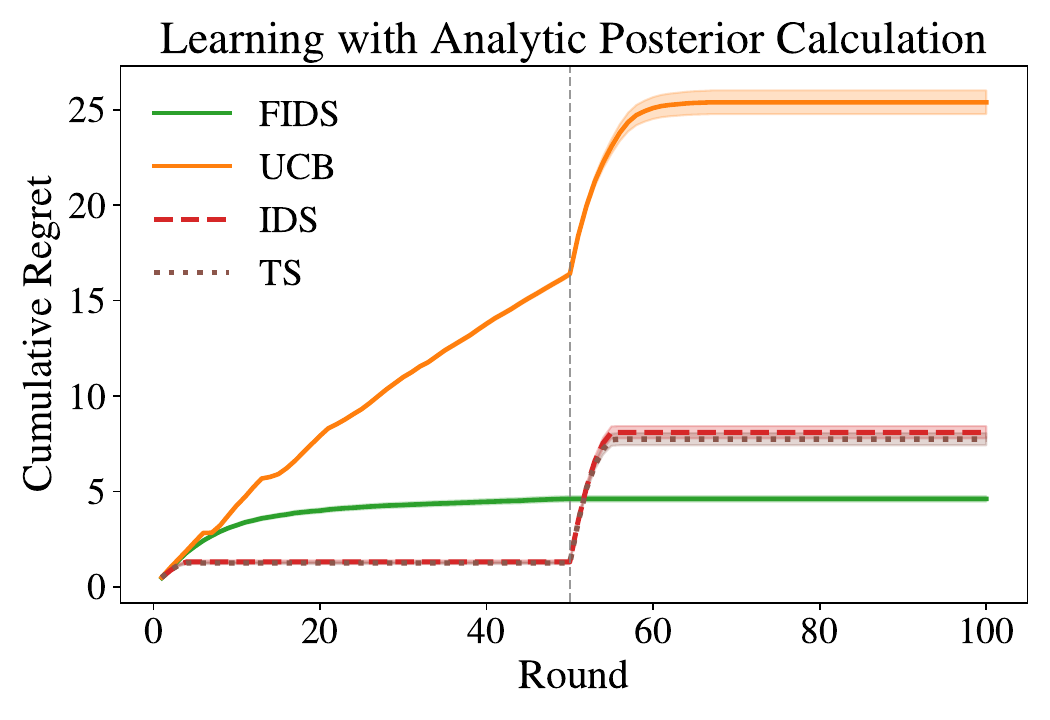}
        \caption{}
        \label{fig:analytic_a}
    \end{subfigure}
    \hfill
    \begin{subfigure}[t]{0.32\linewidth}
        \centering
        \includegraphics[width=\linewidth]{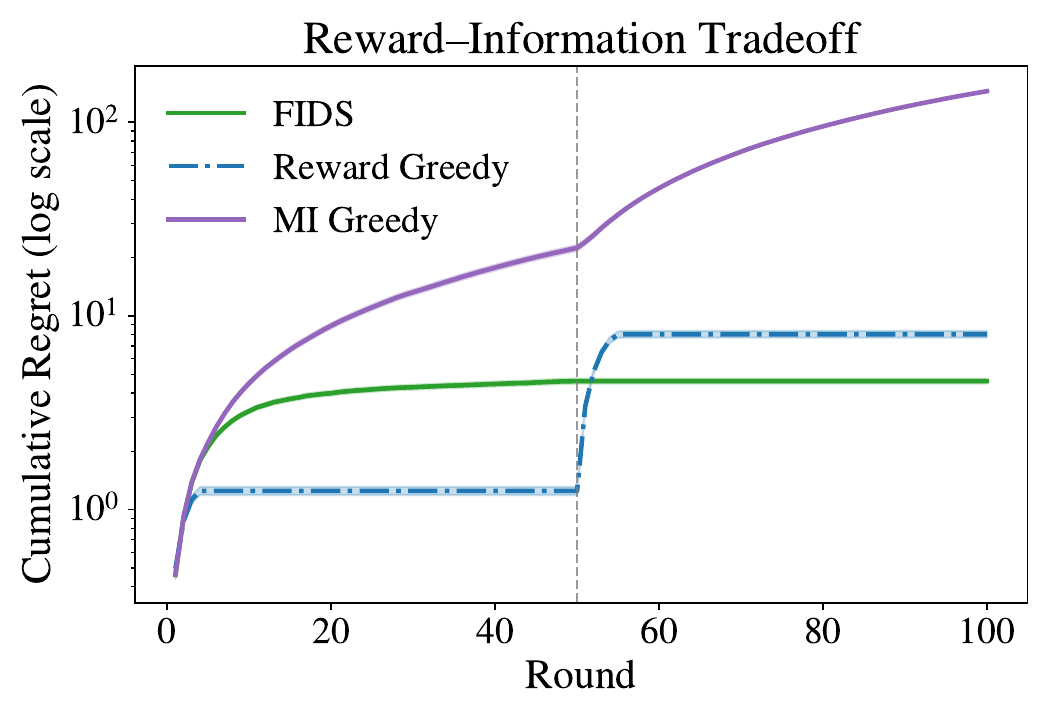}
        \caption{}
        \label{fig:analytic_b}
    \end{subfigure}
    \hfill
    \begin{subfigure}[t]{0.32\linewidth}
        \centering
        \includegraphics[width=\linewidth]{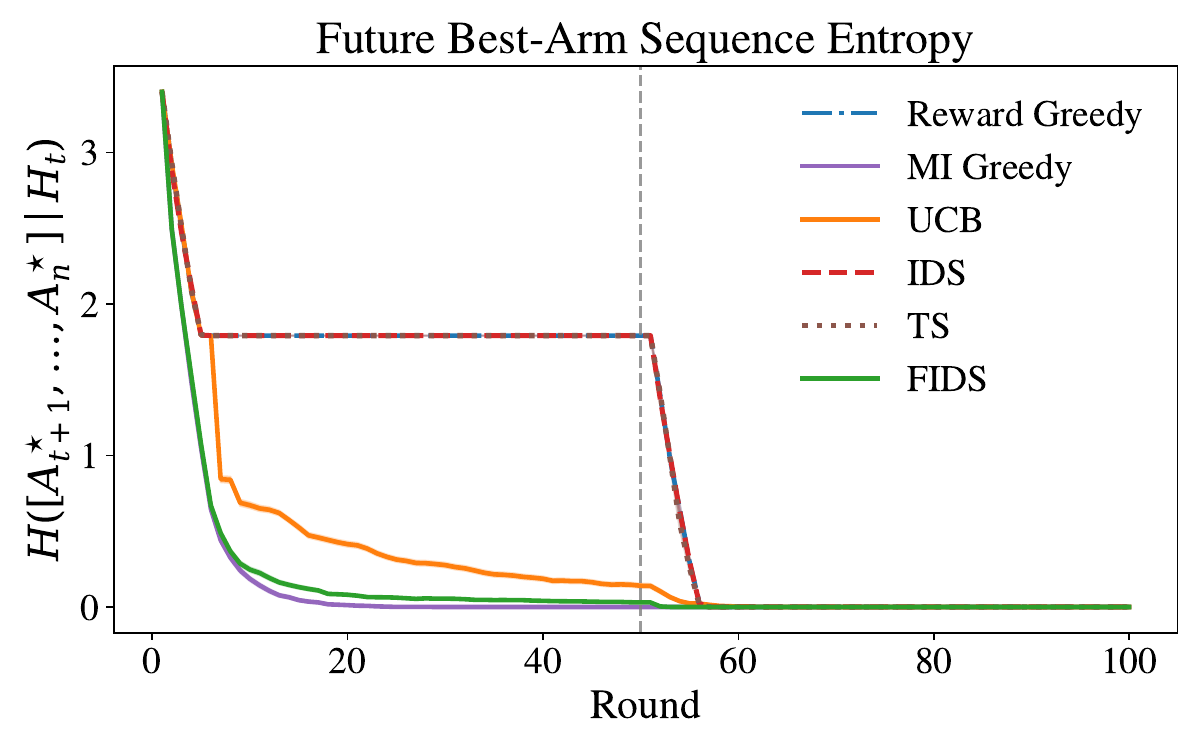}
        \caption{}
        \label{fig:analytic_c}
    \end{subfigure}
    
    \caption{(a) Cumulative regret of FIDS, IDS and TS. (b) Cumulative regret (log scale) of FIDS, Reward Greedy (only optimize for the first reward term of Equation \eqref{eq:equivalent definition of FIDS}), and MI only (only optimize for the second mutual information term of Equation \eqref{eq:equivalent definition of FIDS}). (c) The entropy of \(\left[A^\star_{t+1},\dots, A^\star_n\right]\) conditioned on the history.}
    \label{fig:analytic_env}
    
    \vspace{-10pt}
\end{figure}

\noindent{\bf Numerical results on the magic arm environment.}
We evaluate FIDS, UCB \cite{auer2002finite}, IDS and TS on the magic arm environment.
The results are shown in Figure~\ref{fig:analytic_env}. In Figure~\ref{fig:analytic_a}, compared with the other methods, FIDS spends more rounds identifying the optimal arm in \emph{Env 1}, since it allocates part of its sampling probability to the informative but suboptimal arm \(6\). However, FIDS almost immediately identifies the optimal arm in \emph{Env 2}, resulting in nearly zero regret during the second half of the horizon. This phenomenon is also reflected in Figure~\ref{fig:analytic_c}: by pulling arm \(6\) in \emph{Env 1}, FIDS rapidly reduces the entropy of the optimal arms in future rounds, rather than only the entropy of the optimal arm in the current environment.

In Figure~\ref{fig:analytic_b}, we additionally compare FIDS with two ablated policies: a ``pure exploitation'' policy that optimizes only the reward term in Equation~\eqref{eq:equivalent definition of FIDS}, and a ``pure exploration'' policy that optimizes only the mutual information term. FIDS outperforms both policies by effectively balancing exploration and exploitation.

\section{Inferring FIDS policy from Data via Supervised Learning}\label{sec:SL alg}

Computing $\pi^{\mathrm{FIDS}}$ exactly is often impractical: the prior over environments may be unknown, and even when it is known, evaluating Equation~\eqref{eq:equivalent definition of FIDS} requires posterior inference that is generally intractable. A further difficulty is that the mutual information $I_t(A^\star_{t+1:n}; R_{t,a})$ involves the full sequence of future optimal arms, whose support grows exponentially with the horizon.


To address these challenges, we propose a supervised-learning approach (Algorithm~\ref{alg:fids-offline}). Given an offline dataset of environment instances, we train models to estimate the two quantities needed by FIDS: the expected reward $\mathbb{E}_t[R_{t,a}]$ and a \emph{windowed} approximation $I_t(A^\star_{t+1:t+k}; R_{t,a})$, where $k$  is a fixed lookahead horizon that controls the trade-off between fidelity and computational cost. At deployment, these estimates are plugged into Equation~\eqref{eq:equivalent definition of FIDS} to compute $\pi^{\text{FIDS}}$. We describe the training and deployment phases in detail below.



\begin{algorithm}[t]\label{alg:FIDS_offline}
\caption{Infer \(\pi^{\text{FIDS}}\) from Dataset \(\mathcal{D}\) with prediction window length \(k\): Training and Deployment}
\label{alg:fids-offline}
\footnotesize
\Comment{Training phase - supervised learning with offline dataset \(\mathcal{D}\)}

Construct training samples \(\{z_t^\xi\}_{t=1}^{n}\) from each data instance \(\xi \in \mathcal{D}\); initialize the model \(M_\theta\) with parameter \(\theta\).

\For{\(t = 1\) \KwTo \(N_{\text{epochs}}\)}{
Calculate loss \eqref{eq:training loss} and do back-propagation to update \(\theta\)
}
Output the trained model \(M_\theta\)

\Comment{Deployment phase}

\For{\(t = 1\) \KwTo \(n\)}{

     For each action $a$ compute  \(\widehat{\mathbb{E}}_{t}(R_{t,a})\) and \(\widehat{I}_t(A_{t+1:t+k}^\star;R_{t,a})\) using Eqs. \eqref{eq:estimated reward mean}-\eqref{eq:estimated MI}
     
Solve

    \[
\begin{aligned}
(i^\star, j^\star, \alpha^\star) = \argmax_{\substack{i,j \in \mathcal{A}, \alpha \in [0,1]}}
\Bigl\{
&\alpha \widehat{\mathbb{E}}_{t}(R_{t,i})
+ (1-\alpha)\widehat{\mathbb{E}}_{t}(R_{t,j})
\\
&+
\sigma\sqrt{
2\sigma K \Bigl(
\alpha \widehat{I}_t(A_{t+1:t+k}^\star;R_{t,i})
+
(1-\alpha)\widehat{I}_t(A_{t+1:t+k}^\star;R_{t,j})
\Bigr)
}
\Bigr\}.
\end{aligned}
\]

Output \(\pi_t^{\text{FIDS}}(i^\star) = \alpha^\star\), \(\pi_t^{\text{FIDS}}(j^\star) = 1 - \alpha^\star\), and \(\pi_t^{\text{FIDS}}(a) = 0\) for all \(a \notin \{i^\star, j^\star\}\).

}

    
\end{algorithm}


\subsection{On the Training Phase}
Algorithm~\ref{alg:fids-offline} takes an offline dataset \(\mathcal{D}\) as input. 
Each data instance \(\xi \in \mathcal{D}\) is generated by first sampling an environment \(\nu\) from the prior, and then collecting:
(i) the full reward table of \(\nu\), \(\{r_{t,a}^\nu\}_{t \in [n],\, a \in \mathcal{A}}\), 
(ii) the best-arm sequence of \(\nu\), \((a_1^{\star,\nu}, \dots, a_n^{\star,\nu})\), and 
(iii) an action sequence \((a_1, \dots, a_n)\) generated by some behavior policy.

For each \(\xi \in \mathcal{D}\) and \(t \in [n]\), we construct one training sample: \(z_t^\xi := \left(h_t^\xi,\; a_{t:t+k}^\star,\; \{r_{t,a}^\xi\}_{a \in \mathcal{A}}\right)\), where \(h_t^\xi := (a_1, r_1^\nu, \dots, a_{t-1}, r_{t-1}^\nu)\) denotes the observed history up to round \(t-1\), and \(r_s^\nu := r_{s,a_s}^\nu\) for \(s \in [t]\). \(a_{t:t+k}^{\star,\xi} := (a_{t}^{\star,\nu}, \dots, a_{t+k}^{\star,\nu})\) denotes the future best-arm sequence.

Our goal is to estimate \(\mathbb{E}_t[R_{t,a}]\) and 
\(I_t(A_{t+1:t+k}^\star; R_{t,a})\). 
To this end, we approximate the following conditional distributions: \(\mathbb{P}_t(R_{t,a} \in \cdot)\) for all \(a \in \mathcal{A}\), \(\mathbb{P}_t(A_{t+1:t+k}^\star \in \cdot)\), \(\mathbb{P}_t(A_{t:t+k-1}^\star \in \cdot)\). In practice, we model the reward posterior by a Gaussian, and the best-arm posteriors by categorical distributions. To learn these models, we use a sequential architecture, specifically GPT-2 \cite{radford2019language}. The architecture has three heads, and outputs:
\[
M_\theta^{\mathrm{rew}}(\cdot \mid h_t^\xi, a), \quad 
M_\theta^{\mathrm{fut}+}(\cdot \mid h_t^\xi), \quad 
M_\theta^{\mathrm{fut}}(\cdot \mid h_t^\xi),
\]
which model the conditional distributions of \(R_{t,a}\), \(A_{t+1:t+k}^\star\), and \(A_{t:t+k-1}^\star\), respectively.
 Then, to train the reward posterior, we simply use negative log-likelihood, while the best-arm posteriors are trained via cross-entropy loss. The overall training objective is
\begin{align}
\mathcal{L}_\theta 
= \sum_{\xi \in \mathcal{D}} \sum_{t=1}^{n} 
\Bigg[\sum_{a \in \mathcal{A}} 
\mathrm{NLL}\big(M_\theta^{\mathrm{rew}}(\cdot \mid h_t^\xi, a), r_{t,a}^\xi\big)
+ \, \mathrm{CE}\big(M_\theta^{\mathrm{fut}+}(\cdot \mid h_t^\xi), a_{t+1:t+k}^\star\big)\notag
\\
\qquad
+ \, \mathrm{CE}\big(M_\theta^{\mathrm{fut}}(\cdot \mid h_t^\xi), a_{t:t+k-1}^\star\big)
\Bigg]. \label{eq:training loss}
\end{align}

\subsection{On the Deployment Phase}
At deployment, we use the inferred posteriors to compute FIDS, which relies on \(\mathbb{E}_t\left[R_{t,a}\right]\) and \(I_t\left(A^\star_{t+1:t+k};R_{t,a}\right)\).
Given a history \(h_t\), we first estimate the conditional reward mean by the mean of the learned reward posterior:
\begin{align}
    \widehat{\mathbb{E}}_{t}\left[R_{t,a}\right]
=
\mathbb{E}_{\widetilde{R}_{t,a} \sim M_\theta^{\mathrm{rew}}(\cdot \mid h_t,a)}
\left[\widetilde{R}_{t,a}\right].\label{eq:estimated reward mean}
\end{align}
Since \(M_\theta^{\mathrm{rew}}(\cdot \mid h_t,a)\) is Gaussian, this is simply the predicted Gaussian mean.
We next estimate the conditional mutual information via an entropy decomposition
\[
I_t(A_{t+1:t+k}^\star;R_{t,a})=
H_t(A_{t+1:t+k}^\star)
-
H_t(A_{t+1:t+k}^\star \mid R_{t,a}).
\]
The first term is estimated directly from the future head: \(\widehat{H}_t(A_{t+1:t+k}^\star)
=
H\!\left(M_\theta^{\mathrm{fut}+}(\cdot \mid h_t)\right)\).
For the conditional entropy term, we draw \(L\) Monte Carlo samples: \(\widetilde{r}_{t,a}^{(1)},\dots,\widetilde{r}_{t,a}^{(L)}
\sim
M_\theta^{\mathrm{rew}}(\cdot \mid h_t,a)\).
For each sampled reward \(\widetilde{r}_{t,a}^{(\ell)}\), we form the hypothetical updated history \(\widetilde{h}_{t+1}^{(\ell,a)}
=
(h_t, a, \widetilde{r}_{t,a}^{(\ell)})\). 
Since \(M_\theta^{\mathrm{fut}}(\cdot \mid \widetilde{h}_{t+1}^{(\ell,a)})\) models the posterior distribution of \(A_{t+1:t+k}^\star\) after observing action \(a\) and reward \(\widetilde{r}_{t,a}^{(\ell)}\), we estimate
\[
\widehat{H}_t(A_{t+1:t+k}^\star \mid R_{t,a})
=
\frac{1}{L}\sum_{\ell=1}^L
H\!\left(
M_\theta^{\mathrm{fut}}(\cdot \mid \widetilde{h}_{t+1}^{(\ell,a)})
\right).
\]
Therefore, the plug-in Monte Carlo estimator of the conditional mutual information is
\begin{align}
    \widehat{I}_t(A_{t+1:t+k}^\star;R_{t,a})
=
H\!\left(M_\theta^{\mathrm{fut}+}(\cdot \mid h_t)\right)
-
\frac{1}{L}\sum_{\ell=1}^L
H\!\left(
M_\theta^{\mathrm{fut}}(\cdot \mid \widetilde{h}_{t+1}^{(\ell,a)})
\right).\label{eq:estimated MI}
\end{align}
We choose \(L=16\) for all the experiments in Section \ref{sec:experiments}.

\paragraph{Solving Equation \eqref{eq:equivalent definition of FIDS}.}
With the learned \(\mathbb{E}_t\left[R_{t,a}\right]\) and \(I_t\left(\left[A^\star_{t+1},\dots,A^\star_{t+k}\right];R_{t,a}\right)\), FIDS policy can be obtained by just plugging in the learned two quantities into Eq.~\eqref{eq:equivalent definition of FIDS} and solve the corresponding optimization problem. With Lemma \ref{lemma:2_supp}, for the plugged-in \(\widehat{\mathbb{E}}_{t}(R_{t,a})\) and \(\widehat{I}_t(A_{t+1:n}^\star;R_{t,a})\), one can solve Eq. \eqref{eq:equivalent definition of FIDS} by first enumerating all pairs of arms and using line search to find the optimal weights assigned on the arms. Details are presented in Alg. \ref{alg:fids-offline}.

\section{Experiments}\label{sec:experiments}
We evaluate the FIDS policy learned via Algorithm~\ref{alg:fids-offline} on two environment classes for which exact posterior inference is intractable: the \emph{One-Step Predictive} environment and the \emph{Pair-Revealing} environment. In both environments, information about future optimal arms is encoded in the rewards of a currently suboptimal arm: the first reveals the next-round optimum exactly, while the second only narrows the future best arm to a pair. Further details are provided in the next subsections. Throughout, we set $\sigma = 0.7$ in Equation~\eqref{eq:equivalent definition of FIDS}, which is a valid sub-Gaussian parameter by Lemma~\ref{lemma:env2 subg}.


\noindent{\bf Comparison.}
We compare the following methods: (i) FIDS with prediction window length \(k=5\); (ii) FIDS with \(k=3,1\) (used only in the second Pair-Revealing Environment); (iii) DPT; and (iv) IDS. We train IDS using the same procedure as Algorithm~\ref{alg:fids-offline}. The only difference is that for IDS we estimate \(I_t(A_t^\star; R_{t,a})\) instead of \(I_t(A^\star_{t+1:t+k}; R_{t,a})\).

\subsection{One-Step Predictive Environment}\label{subsec:onestep_predictive_env}
\noindent{\bf Environment Description} (Next round best arm encoding).
We consider a nonstationary 3-arm bandit with horizon $n=100$, divided into 50 independent windows of length 2. For each window $j \in \{1,\dots,50\}$, corresponding to rounds $2j{-}1$ and $2j$, we independently sample $m_j \sim \mathrm{Unif}(\{1,2,3\})$.

At round $2j{-}1$, we sample $Z_j \sim \mathrm{Unif}(\{1,2\})$. Arms 1 and 2 have rewards $\mu_{2j-1,a} = \mathbf{1}[a = Z_j]$, while arm 3 has reward $\mu_{2j-1,3} = 1/(2m_j)$. At round $2j$, the best arm is determined by $m_j$: $\mu_{2j,a} = \mathbf{1}[a = m_j]$ for all $a \in \{1,2,3\}$. All rewards are deterministic.

The key feature is that arm 3's reward in odd rounds encodes the identity of the best arm in the following even round, creating a one-step-ahead information channel that a forward-looking policy can exploit.





\begin{figure}[t]
\centering
\begin{subfigure}[t]{0.4\linewidth}
    \centering
    \includegraphics[width=\linewidth]{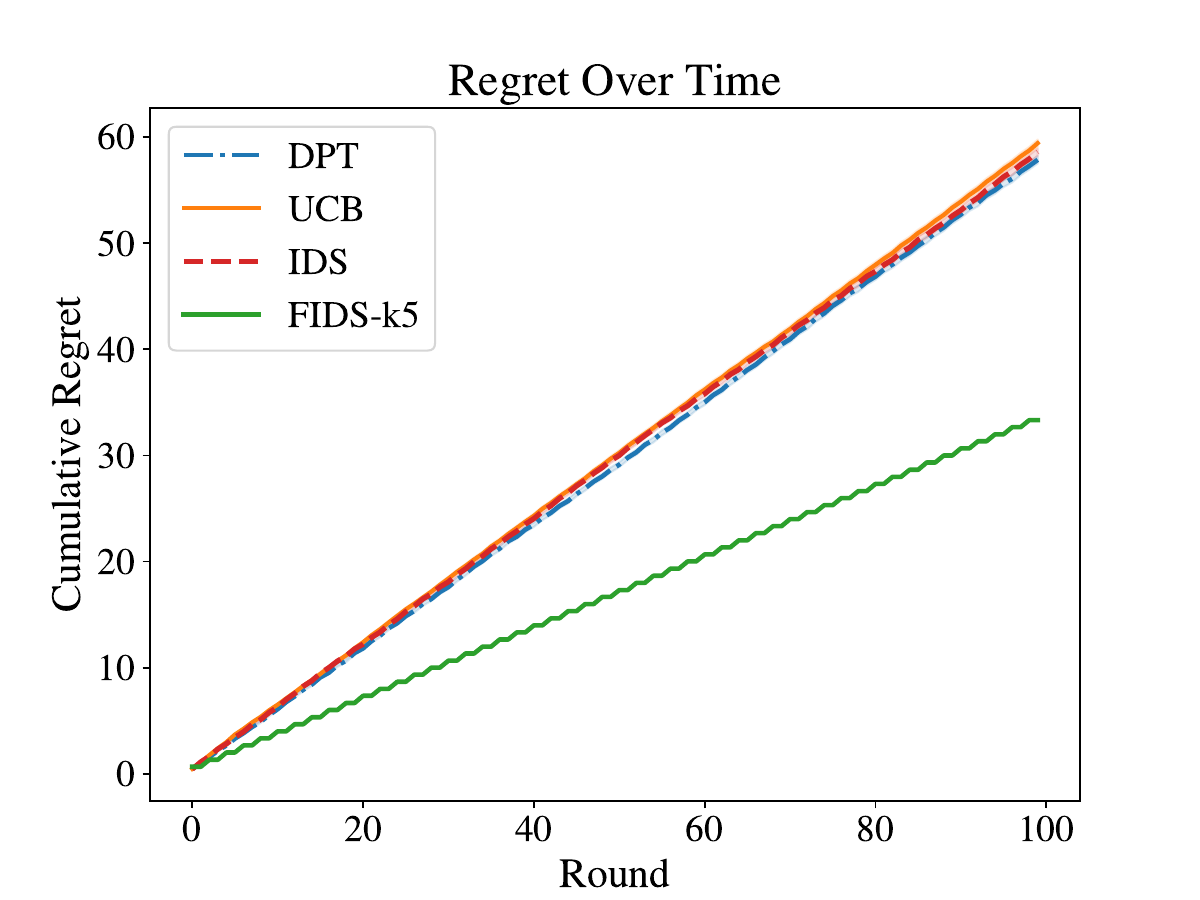}
    \caption{}
    \label{fig:onepredict cum regret}
\end{subfigure}%
\hfill
\begin{subfigure}[t]{0.4\linewidth}
    \centering
    \includegraphics[width=\linewidth]{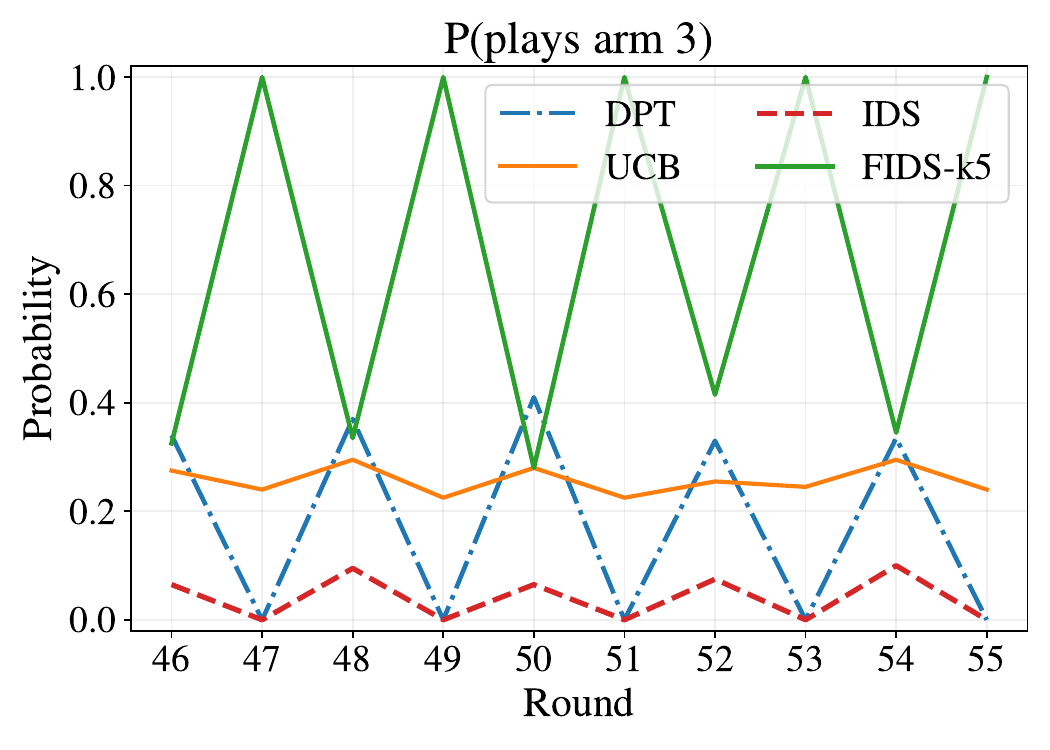}
    \caption{}
    \label{fig:onepredict seg mag}
\end{subfigure}

    \caption{(a) Cumulative Regret of DPT, IDS, UCB and FIDS-k5 in One Step Predictive Environment. (b) The probability of choosing arm 3 in each round. }
    \label{fig:five}
    \vspace{-15pt}
\end{figure}

\noindent{\bf Results.}
Figures~\ref{fig:onepredict cum regret} and~\ref{fig:onepredict seg mag}
summarize the main results (additional plots in
Figures~\ref{fig:onepredict mag}--\ref{fig:onepredict seg best} in the
appendix). Although the windows are independent, each contains a
one-step information channel: pulling arm $3$ in the first round
reveals the optimal arm in the second round. The learned FIDS policy
recovers exactly this behavior: it selects arm $3$ in the first round
of each window and the true optimal arm in the second
(Figures~\ref{fig:onepredict mag},~\ref{fig:onepredict seg mag},
and~\ref{fig:onepredict seg best}). In contrast, TS and IDS never
learn to exploit arm $3$, since their exploration targets only the
current optimal arm.



\subsection{Pair-Revealing Environment}

\begin{figure}[b]
\vspace{-5pt}
    \centering
    
    \begin{subfigure}[t]{0.32\linewidth}
        \centering
        \includegraphics[width=\linewidth]{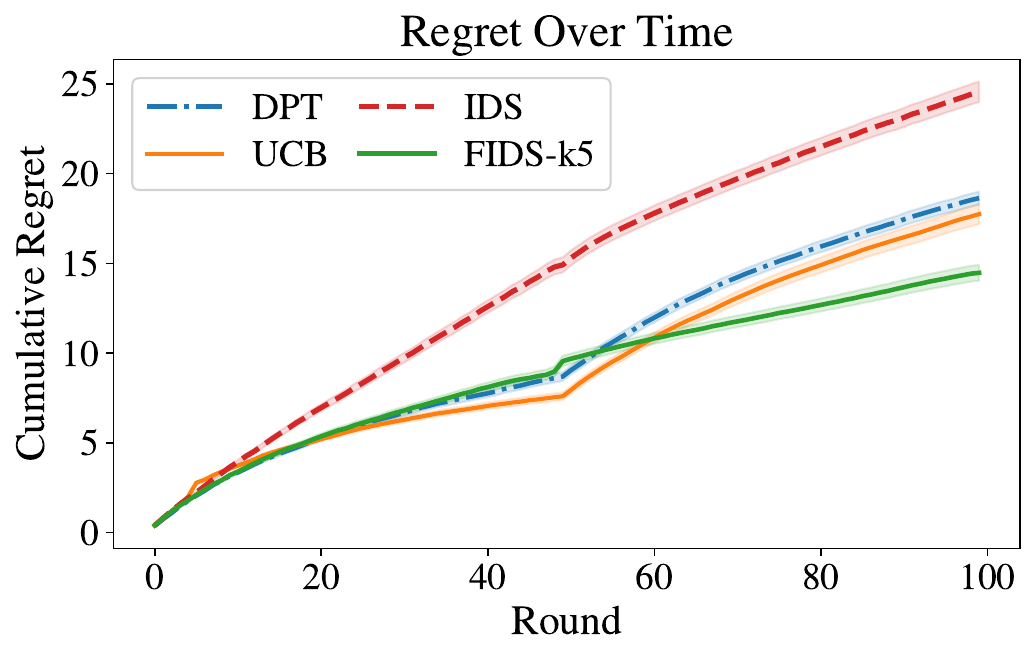}
        \caption{}
        \label{fig:top2_0.0_a}
    \end{subfigure}
    \hfill
    \begin{subfigure}[t]{0.32\linewidth}
        \centering
        \includegraphics[width=\linewidth]{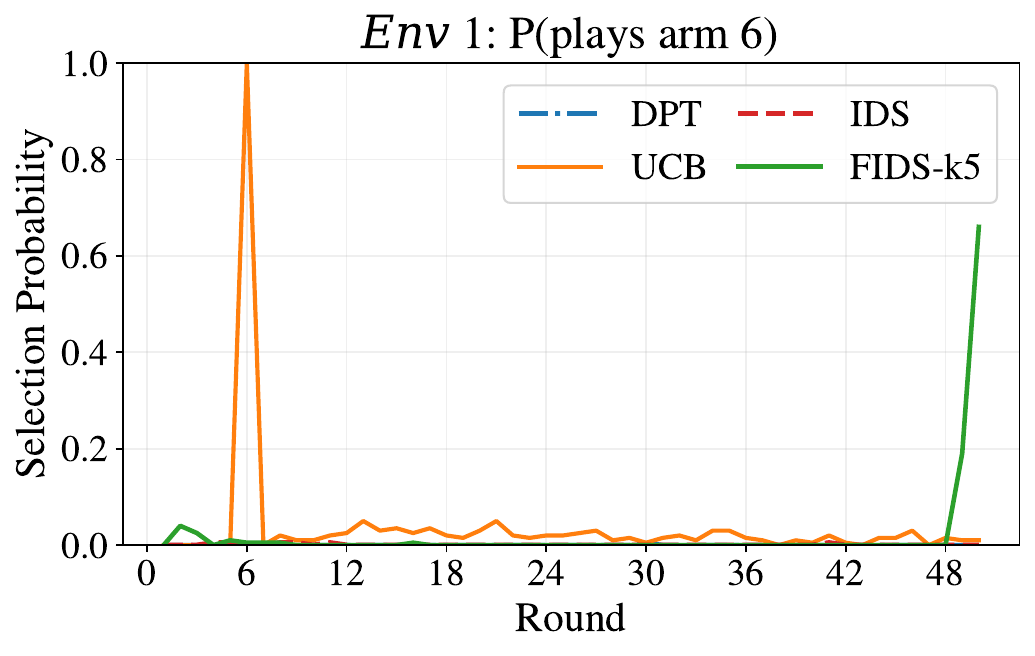}
        \caption{}
        \label{fig:top2_0.0_b}
    \end{subfigure}
    \hfill
    \begin{subfigure}[t]{0.32\linewidth}
        \centering
        \includegraphics[width=\linewidth]{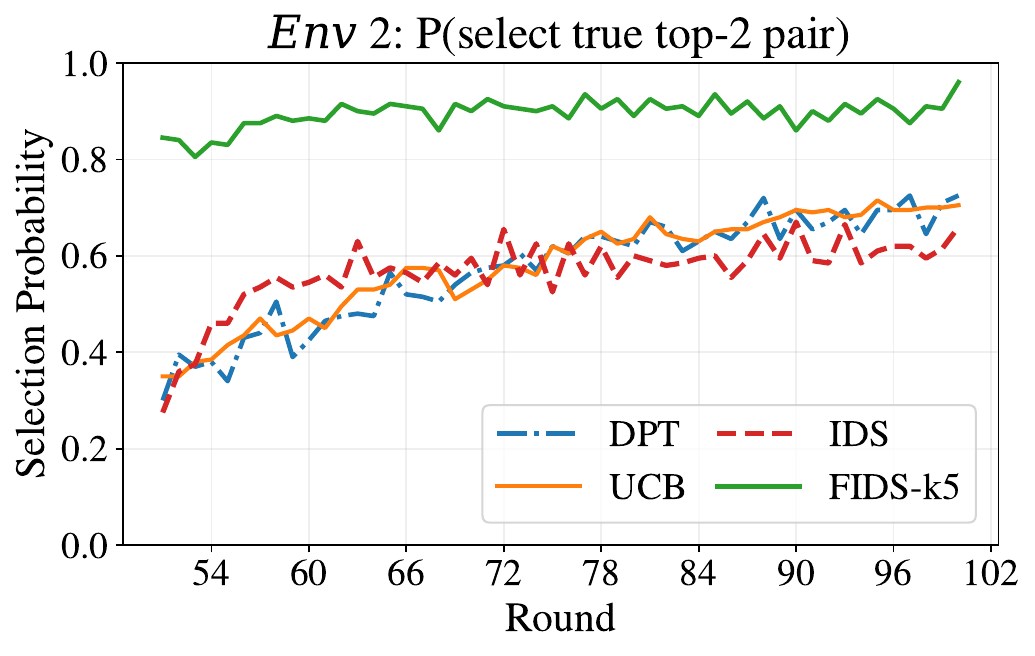}
        \caption{}
        \label{fig:top2_0.0_c}
    \end{subfigure}
    
    \caption{Evaluation Results for the Pair-Revealing Environment with \(\sigma_6^2=0.0\). (a) Cumulative regret.  (b) The probability of play arm 6 at each round. (c) The probability of selecting an arm that belongs to the true top-2 set of the underlying environment.}
    \label{fig:env2_0.0}
\end{figure}

\noindent{\bf Environment Description.}
Unlike the previous environment, where the informative arm directly identifies the future optimum, here it only reveals the top-2 arm pair in the future environment.
We consider a nonstationary 6-arm bandit with horizon $n=100$ and a single change point at $t=50$. The environment is \texttt{Env 1} for $t \le 50$ and \texttt{Env 2} for $t > 50$, stationary within each segment.
\begin{itemize}
    \item In \texttt{Env 1}, let $Z_1 \sim \mathrm{Unif}(\{1,\dots,5\})$. For $a \in \{1,\dots,5\}$, $\mu_a^1 = 1$ if $a = Z_1$ and $\mu_a^1 = U_a$ otherwise, where $U_a \overset{\mathrm{i.i.d.}}{\sim} \mathrm{Unif}([0,1])$, with $R_{t,a} \sim \mathcal{N}(\mu_a^1, 0.5^2)$. Arm 6 has mean $\mu_6^1 = m/30$, where $m \sim \mathrm{Unif}(\{0,\dots,14\})$ indexes one of the $\binom{6}{2}=15$ unordered arm pairs; we denote this pair by $\mathcal{S}(m)$. We consider two noise settings for arm 6: $\sigma_6^2 = 0$ (exact observation) and $\sigma_6^2 = 0.05^2$ (noisy observation).

\item In \texttt{Env 2}, we sample $V_1,\dots,V_6 \overset{\mathrm{i.i.d.}}{\sim} \mathrm{Unif}([0,1])$ and assign the two largest values to the arms in $\mathcal{S}(m)$, with the remaining values randomly permuted among the other arms, yielding means $\{\mu_a^2\}$. Rewards satisfy $R_{t,a} \sim \mathcal{N}(\mu_a^2, 0.5^2)$.
\end{itemize}

The information structure is that arm 6 in \texttt{Env 1}, while never optimal, encodes which pair of arms will dominate in \texttt{Env 2}. The two noise settings let us examine how the precision of this signal affects the benefit of future information exploration.

\noindent{\bf Results.} 
We evaluate under two noise settings for arm $6$.

\begin{itemize}
    \item ($\sigma_6^2 = 0$).
We set the prediction window to $k=5$. Figure~\ref{fig:top2_0.0_a}
shows that FIDS matches the baselines in \texttt{Env 1}, while
intentionally pulling arm $6$ (Figure~\ref{fig:top2_0.0_b}) to gather
information about \texttt{Env 2}. Upon entering \texttt{Env 2}, FIDS
identifies the candidate top-2 arms and focuses exploration on them
rather than the full arm set.

\item ($\sigma_6^2 = 0.05^2$).
When arm $6$'s reward is noisy, recovering $\mu_6^1$ requires multiple
pulls, making a longer prediction window more valuable. We compare
FIDS with $k=5$ (FIDS-K5) and a shorter window. As shown in
Figures~\ref{fig:top2_0.05_b} and~\ref{fig:top2_0.05_c}, pulling arm
$6$ earlier in \texttt{Env 1} allows FIDS-K5 to infer the top-2 pair
more accurately, leading to faster adaptation after the change point.

\end{itemize}

Across both environments, the learned FIDS policy consistently
exploits informative but suboptimal arms that TS and IDS ignore,
confirming that Algorithm~\ref{alg:fids-offline} successfully recovers
future information exploration behavior from  data.



\begin{figure}[t]
    \centering
    
    \begin{subfigure}[t]{0.32\linewidth}
        \centering
        \includegraphics[width=\linewidth]{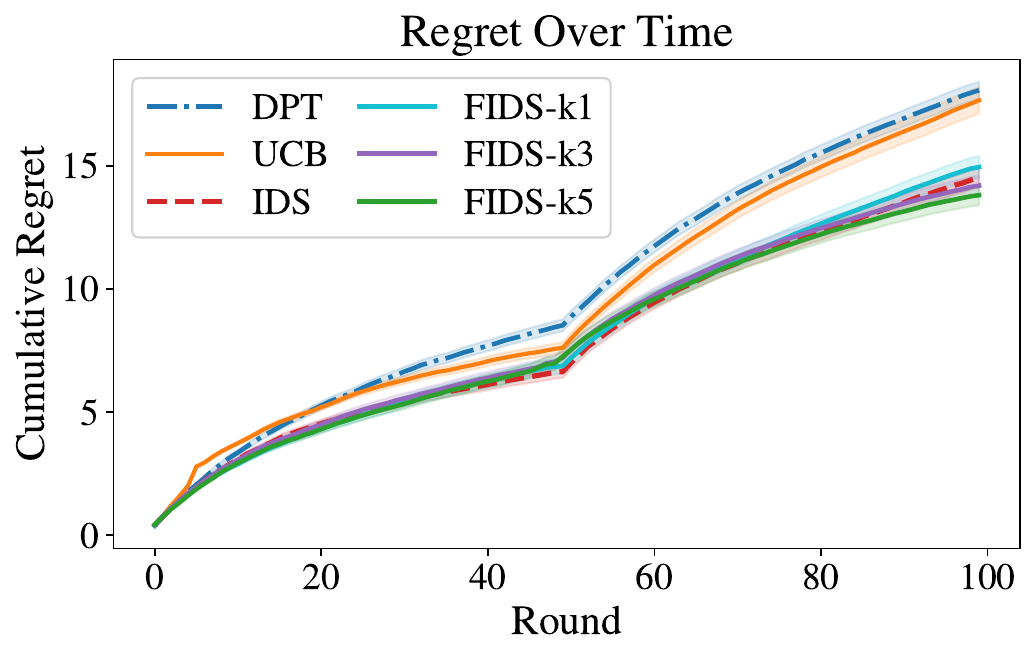}
        \caption{}
        \label{fig:top2_0.05_a}
    \end{subfigure}
    \hfill
    \begin{subfigure}[t]{0.32\linewidth}
        \centering
        \includegraphics[width=\linewidth]{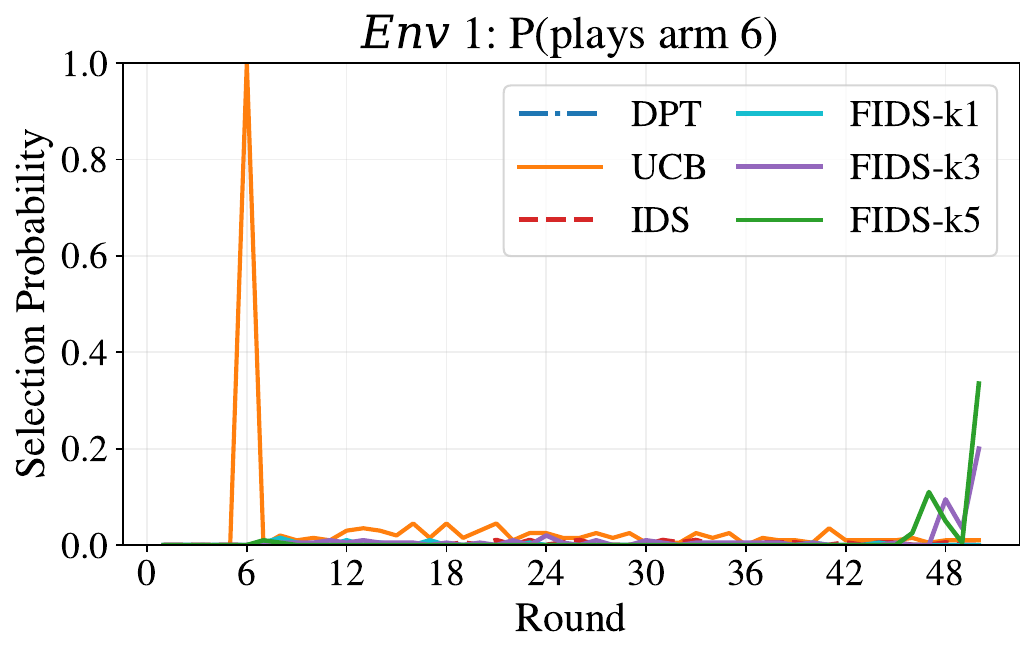}
        \caption{}
        \label{fig:top2_0.05_b}
    \end{subfigure}
    \hfill
    \begin{subfigure}[t]{0.32\linewidth}
        \centering
        \includegraphics[width=\linewidth]{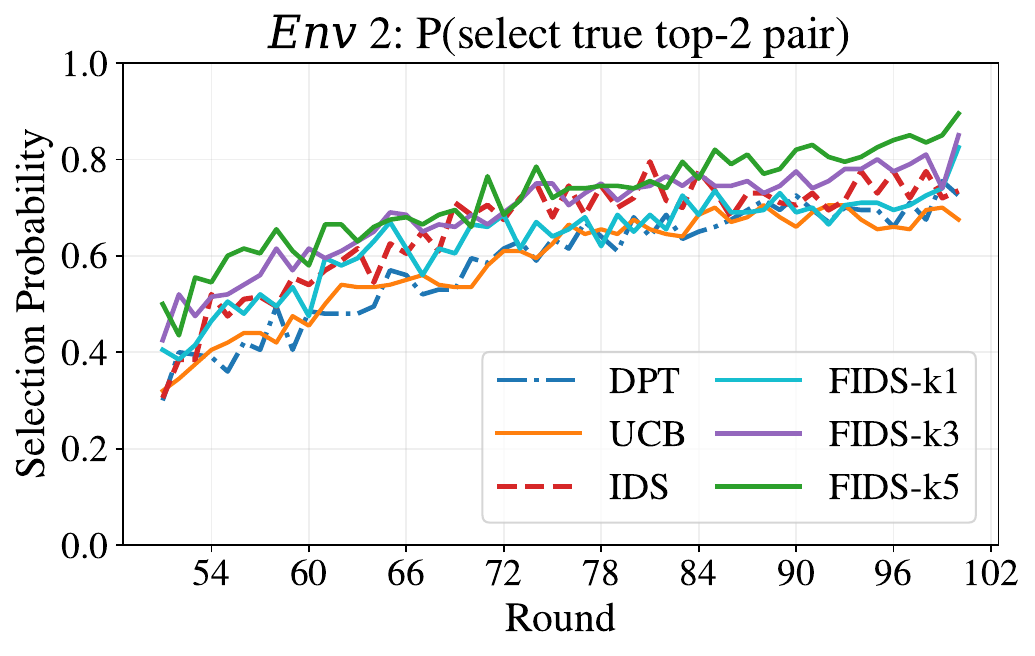}
        \caption{}
        \label{fig:top2_0.05_c}
    \end{subfigure}
    
    \caption{Evaluation Results for the Pair-Revealing Environment with \(\sigma_6^2=0.05^2\). (a) Cumulative regret.  (b) The probability of play arm 6 at each round. (c) The probability of selecting an arm that belongs to the true top-2 set of the underlying environment.}
    \label{fig:env2_0.05}
    \vspace{-5pt}
\end{figure}

\section{Towards Full FIDS Objective Training}\label{sec:full fids}
A major limitation of Algorithm \ref{alg:fids-offline} is that, due to the presence of \(k\), the size of the truncated window, it does not optimize the exact FIDS objective. This is because Algorithm \ref{alg:fids-offline} estimates \(H_t\left(A^\star_{t+1:n}\right)\) and \(H_t\left(A^\star_{t+1:n}\mid R_{t,a}\right)\) by first learning \(\mathbb{P}_t\left(A^\star_{t+1:n} \in \cdot\right)\) and \(\mathbb{P}_{t+1}\left(A^\star_{t+1:n}\in \cdot\right)\). This amounts to a \(|\mathcal{A}|^{n-t}\)-class classification problem and is therefore infeasible except for small values of \(n\). To address this issue, in this section we propose a new method to directly learn \(H_t\left(A^\star_{t+1:n}\right)\) and \(H_t\left(A^\star_{t+1:n}\mid R_{t,a}\right)\)thereby bypassing the aforementioned challenge. We describe the training procedures for learning these two quantities below.

\noindent{\bf On learning \(H_t\left(A^\star_{t+1:n}\right)\).} The training procedure consists of two stages, in which we train two different models, denoted by \(M_{\theta_1}\) and \(M_{\theta_2}\), respectively. \textit{Stage 1}: We train a model \(M_{\theta_1}\left(\cdot \mid H_{t},A^\star_{t+1,\dots, \ell}\right)\) that outputs a probability distribution over \(\mathcal{A}\), represented as a vector in \(\mathbb{R}^{|\mathcal{A}|}\), and predicts the conditional distribution \(\mathbb{P}_t\left(A^\star_{\ell+1} \in \cdot \mid A^\star_{t+1,\dots,\ell}\right)\) using the loss function in (\ref{loss:theta_1}). \textit{Stage 2}: The model \(M_{\theta_2}\) outputs a scalar that predicts \(H_t\left(A^\star_{t+1;n}\right)\) using the loss function in (\ref{loss:theta_2}).

\begin{align}
    \mathcal{L}_{\theta_1} = \sum_{\xi \in \mathcal{D}}\sum_{t=1}^n\sum_{\ell=t}^{n-1} \text{CE}\left(M_{\theta_1}\left(\cdot \mid h^\xi_{t},a^{\star,\xi}_{t+1,\dots, \ell}\right), a^{\star,\xi}_{\ell+1}\right) \label{loss:theta_1}
\end{align}

\begin{align}
    \mathcal{L}_{\theta_2} = \sum_{\xi \in \mathcal{D}}\sum_{t=1}^n \left(M_{\theta_2}\left(\cdot \mid h^\xi_t\right) - \sum_{\ell=t}^{n-1}\text{Entropy}\left( M_{\theta_1} \left(\cdot \mid h^\xi_t, a^\star_{t+1:\ell}\right)\right)\right)^2 \label{loss:theta_2}
\end{align}

\noindent{\bf On learning \(H_t\left(A^\star_{t+1:n} | R_{t,a}\right)\)}. Similarly, the training procedure consists of two stages, in which we train two models, denoted by \(M_{\phi_1}\) and \(M_{\phi_2}\), respectively. \textit{Stage 1}: We train a model \(M_{\phi_1}\left(\cdot \mid H_{t+1}, A^\star_{t+1:\ell}\right)\) that outputs a probability distribution over \(\mathcal{A}\) , represented as a vector in \(\mathbb{R}^{|\mathcal{A}|}\), and predicts the conditional distribution \(\mathbb{P}_{t+1}\left(A^\star_{\ell+1} \in \cdot \mid A^\star_{t+1,\dots,\ell}\right)\) using the loss function in (\ref{loss: phi_1}). \textit{Stage 2}: The model \(M_{\phi_2}\) outputs a vector in \(\mathbb{R}^{|\mathcal{A}|}\), where its \(a\)-th element, \( M_{\phi_2}\left(\cdot \mid H_t\right)[a]\), predicts \(H_t\left(A^\star_{t+1;n} \mid R_{t,a}\right)\) using the loss function in (\ref{loss: phi_2}).

\begin{align}
    \mathcal{L}_{\phi_1} = \sum_{\xi \in \mathcal{D}}\sum_{t=1}^n\sum_{\ell=t}^{n-1} \text{CE}\left(M_{\phi_1}\left(\cdot \mid h^\xi_{t+1},a^\star_{t+1,\dots, \ell}\right), a^\star_{\ell+1}\right) \label{loss: phi_1}
\end{align}

\begin{align}
    \mathcal{L}_{\phi_2} = \sum_{\xi\in \mathcal{D}}\sum_{t=1}^n\sum_{a \in \mathcal{A}} \mathbb{I}\left\{a^\xi_t=a\right\}/b_t(a)\cdot \left(M_{\phi_2}\left(\cdot\mid h_t^\xi\right)[a]-\sum_{\ell=t}^{n-1}\text{Entropy}\left(M_{\phi_1}\left(\cdot \mid h^\xi_t, a_t^\xi, r_t^\xi, a^{\star,\xi}_{t+1:\ell}\right)\right)\right)^2 \label{loss: phi_2}
\end{align}
where \(b^\xi_t\) denotes the behavior policy used to collect trajectory \(\xi\), and \(b^\xi_t(a)\) denotes the probability of choosing action \(a\) under \(b^\xi_t\).

One can show that, if \(M_{\theta_1}\) and \(M_{\theta_2}\) (resp. \(M_{\phi_1}\) and \(M_{\phi_2}\)) are global minimizers of (\ref{loss:theta_1}) and (\ref{loss:theta_2}) (resp. (\ref{loss: phi_1}) and (\ref{loss: phi_2})), then \(M_{\theta_2}\) (resp. \(M_{\phi_2}\)) perfectly recovers \(H_t\left(A^\star_{t+1:n}\right)\) (resp. \(H_t\left(A^\star_{t+1:n}\mid R_{t,a}\right)\)) under appropriate realizability assumptions on the model function classes and in the limit of infinite training data.

After training, we estimate the mutual information \(I_t\left(A^\star_{t+1:n};R_{t,a}\right)\) by calculating \(M_{\theta_2}\left(\cdot \mid H_t\right) - M_{\phi_2}\left(\cdot \mid H_t\right)[a]\).

\section{Conclusion}
We introduced FIDS, an algorithm for Bayesian nonstationary bandits
that balances instantaneous reward against information gained about
future optimal arms. FIDS matches the regret guarantee of Thompson
Sampling, yet analytical examples show it can significantly outperform
both TS and IDS by exploiting predictive structure across environments.
To handle settings where exact posterior inference is intractable, we
proposed a supervised-learning framework that learns the FIDS
policy with a truncated-window approximation from offline data. Experiments confirm that the learned policy
consistently outperforms supervised-learning baselines with alternative
exploration objectives. To remove the truncated-window approximation, we propose an alternative training pipeline in Section \ref{sec:full fids} that recovers the full FIDS policy at optimality. The empirical evaluation of this training method remains an important direction for future work.




\bibliographystyle{plainnat}
\bibliography{references}

\newpage

\appendix

\section{More Related Work}
\paragraph{Nonstationary Bandit Learning.}
Several variants of Thompson Sampling and UCB have been proposed for nonstationary bandits, including \cite{trovo2020sliding,gupta2011thompson,mellor2013thompson,kocsis2006discounted,garivier2008upper}. In the frequentist setting, recent works \cite{abbasi2023new,suk2022tracking} establish regret bounds of order \(\tilde{O}(\sqrt{ST})\), where \(S\) denotes the number of changes in the identity of the optimal arm. This dependence is closely related to our regret bound in Theorem~\ref{thm:regret}, as both results adapt to the temporal variation of the optimal-arm sequence. Indeed, Theorem 4.3 of \cite{min2023information} implies that, under suitable conditions, our regret bound recovers the same \(\tilde{O}(\sqrt{ST})\) rate. However, frequentist approaches generally do not leverage prior information and therefore may fail to exploit rich structural information encoded in the prior. This distinction becomes particularly important in environments with informative prior structure, such as the examples presented in Section~\ref{sec:examples}.

Predictive Sampling \cite{liu2023nonstationary} studies the same Bayesian nonstationary setting and is motivated by a similar insight: exploration strategies should account for future information in nonstationary environments. Our work differs in that it adopts a stronger regret notion in the theoretical analysis. Consequently, our regret guarantees directly imply guarantees under their framework, whereas the converse implication does not hold.

\paragraph{Information-Directed Methods in Settings beyond Bandits.}
The idea of optimizing information-ratio-based objectives has been extended to a variety of settings, including graph feedback \cite{hao2022contextual}, sparse linear bandits \cite{hao2021information}, linear partial monitoring \cite{kirschner2020information}, and MDPs \cite{hao2022regret}. These works typically define an information ratio tailored to the underlying problem structure and select actions or policies that minimize this quantity. Despite the diversity of these settings, they all assume a stationary underlying environment.

\paragraph{Learning Decision-Making Algorithms from Data.}
Our supervised-learning-based framework for approximating the FIDS policy is primarily inspired by \cite{lee2023supervised}, which introduces the Decision-Pretrained Transformer (DPT), a framework that learns to imitate the Thompson Sampling policy through supervised learning. More broadly, the idea of learning sequential decision-making algorithms from offline trajectories using Transformer-based architectures has been explored in several recent works, including Algorithm Distillation \cite{laskin2022context}, Decision Transformer \cite{chen2021decision}, and Multi-Game Decision Transformer \cite{lee2022multi}.

\section{Information Theoretic Preliminaries}\label{sec:information preliminaries}
In this section, \(X\) is limited to discrete random variables and \(Y, Z\) can be general random variables, under measure \(\mathbb{P}\).
\begin{lemma}(Mutual Information and Entropy, Theorem 2.4.1 in \cite{cover1999elements})\label{lemma:info and entropy}
    \begin{align*}
        H(X) = - \sum_{x} \mathbb{P}\left(X=x\right)\log\left(\mathbb{P}\left(X=x\right)\right)
    \end{align*}
    \begin{align*}
        I(X;Y) = H(X) - H(X|Y)
    \end{align*}
\end{lemma}
\begin{lemma}(Chain Rule for Entropy, Theorem 2.5.1 in \cite{cover1999elements})\label{lemma:chail rule entropy}
    \begin{align*}
        H(X_1, X_2, \dots, X_n) = \sum_{i=1}^n H(X_i \mid X_1,\dots,X_{i-1})
    \end{align*}
\end{lemma}

\begin{lemma}(Chain Rule for Mutual Information, Theorem 2.5.2 in \cite{cover1999elements})\label{lemma:chain rule info}
    \begin{align*}
        I(\left[X_1,X_2,\dots,X_n\right];Y) = \sum_{i=1}^n I(X_i;Y \mid X_1,\dots,X_{i-1})
    \end{align*}
\end{lemma}

A direct application of the above lemma gives the following Corollary,
\begin{corollary}\label{corollary:chain rule info}
    \begin{align*}
        I(X;Y)- I(Z;Y)= I(X;Y|Z) - I(Z;Y|X)
    \end{align*}
\end{corollary}
\begin{proof}
    The proof is done by simply noticing that \(I(X;Y) + I(Z;Y|X) = I(Z;Y)+I(X;Y|Z)=I([X,Z];Y)\) by Lemma \ref{lemma:chain rule info}.
\end{proof}
\begin{corollary}\label{corollary:conditional entropy}
    \begin{align*}
        H(X|Y,Z) \leq H(X|Y)
    \end{align*}
\end{corollary}
\begin{proof}
    \begin{align*}
        H(X|Y,Z) \leq H(X|Y) \Leftrightarrow H(X) - I(X;[Y,Z]) \leq H(X) - I(X;Y) \Leftrightarrow I(X;[Y,Z]) \geq I(X;Y)
    \end{align*}
    which is true by Lemma \ref{lemma:chain rule info}.
\end{proof}

\begin{lemma}(Non-negativity of Mutual information and Entropy)\label{lemma:non-neg}
    \begin{align*}
        H(X)\geq 0; H(X|Y)\geq 0; I(X;Y)\geq 0; I(X;Y|Z) \geq 0
    \end{align*}
\end{lemma}

\begin{lemma}\label{lemma:root difference}
    Let \(a,b \geq 0\), \(\sqrt{a}-\sqrt{b} \leq \sqrt{|a-b|}\)
\end{lemma}
\begin{proof}
    Clearly we only need to consider the case where \(a \geq b\). In such a case,
    \begin{align*}
        \sqrt{a}-\sqrt{b} \leq \sqrt{|a-b|} \Leftrightarrow a+b-2\sqrt{ab} \leq a-b \Leftrightarrow a \geq b
    \end{align*}
\end{proof}

\section{Supporting Lemmas in the Proof of Theorem \ref{thm:regret}}\label{sec:supporting lemmas}
\begin{lemma}\label{lemma:transformation}
    \begin{align}
    \mathbb{E}_t\left[R_{t,A_t^\star} - R_{t,A_t}\right] &= \sum_{a \in [K]}\mathbb{P}_t(A_t=a)\mathbb{E}_t\left[R_{t,A_t^\star} - R_{t,a} \mid A_t=a\right]\notag\\
    &\stackrel{(a)}{=} \sum_{a \in [K]}\pi_t(a)\mathbb{E}_t\left[R_{t,A_t^\star} - R_{t,a}\right]\label{eq:expect transformation}
\end{align}
(a) uses the fact that conditioned on \(\mathcal{F}_{t-1}\), \(A_t\) is jointly independent of \(R_{t,A_t^\star}\) and \(R_{t,a}\).
\begin{align}
    I_t\left(A_t^\star; \left(A_t, R_{t,A_t}\right)\right) &\stackrel{(a)}{=} I_t\left(A_t^\star; A_t\right) + I_t\left(A_t^\star; R_{t,A_t} \mid A_t\right)\notag\\
    &\stackrel{(b)}{=} I_t\left(A_t^\star; R_{t,A_t} \mid A_t\right)\notag\\
    &= \sum_{a \in [K]} \mathbb{P}_t(A_t = a) I_t\left(A_t^\star; R_{t,a} \mid A_t=a\right)\notag\\
    &\stackrel{(c)}{=} \sum_{a \in [K]} \mathbb{P}_t(A_t = a) I_t\left(A_t^\star; R_{t,a}\right)\notag\\
    &= \sum_{a \in [K]} \pi_t(a) I_t\left(A_t^\star; R_{t,a}\right)\label{eq:information transformation}
\end{align}
(a) uses Lemma \ref{lemma:chain rule info}; (b) uses the fact conditioned on \(\mathcal{F}_{t-1}\), \(A_t\) is independent of \(A_t^\star\), and the mutual information between two independent variables is 0; (c) uses the fact that conditioned on \(\mathcal{F}_{t-1}\), \(A_t\) is jointly independent of \(A_t^\star\) and \(R_{t,a}\).

Similarly,
\begin{align}
    I_t\left(\left[A_{t+1}^\star,\dots,A_n^\star\right]; \left(A_t, R_{t,A_t}\right)\right)
    = \sum_{a \in [K]} \pi_t(a) I_t\left(\left[A_{t+1}^\star,\dots,A_n^\star\right]; R_{t,a}\right)\label{eq:future information transformation}
\end{align}
\end{lemma}

\begin{lemma}\label{lemma:bound term 2}
    \begin{align}
    \mathbb{E}\left[\sum_{t=1}^n \sigma\sqrt{2K\cdot I_t(A^\star_{t+1:n};(A_t, R_{t,A_t}))}\right] &\stackrel{(a)}{\leq} \sum_{t=1}^n \sigma\sqrt{2K\cdot \mathbb{E}\left[I_t(A^\star_{t+1:n};(A_t, R_{t,A_t}))\right]}\notag\\
    &= \sum_{t=1}^n \sigma\sqrt{2K\cdot I(A^\star_{t+1:n};(A_t, R_{t,A_t})|\mathcal{F}_{t-1})}\notag\\
    &\stackrel{(b)}{\leq} \sigma\sqrt{2n\cdot K \cdot \sum_{t=1}^n I(A^\star_{t+1:n};(A_t, R_{t,A_t})|\mathcal{F}_{t-1})}\notag\\
    &\stackrel{(c)}{\leq} \sigma\sqrt{2n\cdot K \cdot \sum_{t=1}^n I(A^\star_{t+1:n};(A_t, R_{t,A_t})|\mathcal{F}_{t-1})}\notag\\
    &\stackrel{(d)}{\leq} \sigma\sqrt{2n\cdot K \cdot I(A^\star_{1:n};\mathcal{F}_n)}\notag\\
    &\stackrel{(e)}{\leq} \sigma\sqrt{2n\cdot K \cdot H(A^\star_{1:n})} \label{eq:term2}
\end{align}
(a) uses Jensen's Inequality; (b) uses Cauchy–Schwarz inequality; (c) uses Lemma \ref{lemma:chain rule info} and Lemma \ref{lemma:non-neg}; (d) uses Lemma \ref{lemma:chain rule info}; (e) uses Lemma \ref{lemma:info and entropy} and \ref{lemma:non-neg}.
\end{lemma}

\begin{lemma}\label{lemma:bound term 1}
    \begin{align*}
        \sigma\sqrt{2K\cdot g_t\left(\pi^{\text{TS}}_t\right)} - \sigma\sqrt{2K\cdot \tilde{g}_t\left(\pi^{\text{TS}}_t\right)} \leq \sigma\sqrt{2K}\cdot \sqrt{I_t(A^\star_t;(\tilde{A}_t, R_{t,\tilde{A}_t})|A^\star_{t+1:n})}
    \end{align*}
\begin{proof}
    When \(I_t(A^\star_t;(\tilde{A}_t, R_{t,\tilde{A}_t})) \leq I_t(A^\star_{t+1:n};(\tilde{A}_t, R_{t,\tilde{A}_t}))\). We have \(\sigma\sqrt{2K\cdot g_t\left(\pi^{\text{TS}}_t\right)} - \sigma\sqrt{2K\cdot \tilde{g}_t\left(\pi^{\text{TS}}_t\right)} \leq 0\) directly.
    When \(I_t(A^\star_t;(\tilde{A}_t, R_{t,\tilde{A}_t})) > I_t(A^\star_{t+1:n};(\tilde{A}_t, R_{t,\tilde{A}_t}))\). Then
    \begin{align*}
    \sigma\sqrt{2K\cdot g_t\left(\pi^{\text{TS}}_t\right)} - \sigma\sqrt{2K\cdot \tilde{g}_t\left(\pi^{\text{TS}}_t\right)} &\stackrel{(a)}{\leq} \sigma\sqrt{2K}\cdot \sqrt{I_t(A^\star_t;(\tilde{A}_t, R_{t,\tilde{A}_t})) - I_t(A^\star_{t+1:n};(\tilde{A}_t, R_{t,\tilde{A}_t}))}\\
    &\stackrel{(b)}{=} \sigma\sqrt{2K} \cdot \sqrt{I_t(A^\star_t;(\tilde{A}_t, R_{t,\tilde{A}_t})|A^\star_{t+1:n}) - I_t(A^\star_{t+1:n};(\tilde{A}_t, R_{t,\tilde{A}_t})|A_t^\star)}\\
    &\leq \sigma\sqrt{2K}\cdot \sqrt{I_t(A^\star_t;(\tilde{A}_t, R_{t,\tilde{A}_t})|A^\star_{t+1:n})}
\end{align*}
(a) holds because of Lemma \ref{lemma:root difference}; (b) holds because of Lemma \ref{corollary:chain rule info}.

We conclude the proof by combining the two cases.
\end{proof}
\end{lemma}

\begin{lemma}\label{lemma:sum term1}
    \begin{align}
    \sigma\sqrt{2K}\cdot\mathbb{E}\left[\sum_{t=1}^n \sqrt{g_t\left(\pi^{\text{TS}}_t\right)} - \sqrt{\cdot \tilde{g}_t\left(\pi^{\text{TS}}_t\right)}\right] &\stackrel{(a)}{\leq} \sigma\sqrt{2K}\cdot \mathbb{E}\left[\sum_{t=1}^n\sqrt{I_t(A^\star_t;(\tilde{A}_t, R_{t,\tilde{A}_t})|A^\star_{t+1:n})}\right]\notag\\
    &\leq \sigma\sqrt{2K}\cdot \mathbb{E}\left[\sum_{t=1}^n\sqrt{H_t(A^\star_t|A^\star_{t+1:n})}\right]\notag\\
    &\leq \sigma\sqrt{2K}\cdot \sum_{t=1}^n\sqrt{\mathbb{E}\left[H_t(A^\star_t|[A^\star_{t+1},\dots, A^\star_n]\right])}\notag\\
    &= \sigma\sqrt{2K}\cdot \sum_{t=1}^n\sqrt{H(A^\star_t|([A^\star_{t+1},\dots, A^\star_n],\mathcal{F}_{t-1}))}\notag\\
    &\stackrel{(b)}{\leq} \sigma\sqrt{2K}\cdot \sum_{t=1}^n\sqrt{H(A^\star_t|[A^\star_{t+1},\dots, A^\star_n])}\notag\\
    &\leq \sigma\sqrt{2K}\cdot \sqrt{n\cdot \sum_{t=1}^nH(A^\star_t|[A^\star_{t+1},\dots, A^\star_n])}\notag\\
    &\stackrel{(c)}{=} \sigma\sqrt{2K \cdot n\cdot H([A^\star_{1},\dots, A^\star_n])}\label{eq:term1}
\end{align}
(a) uses Lemma \ref{lemma:bound term 1}; (b) uses Corollary \ref{corollary:conditional entropy}; (c) uses Lemma \ref{lemma:chail rule entropy}.
\end{lemma}

\section{Numerical Results}
Here we report additional numerical results for the one-step predictive environment in Sec.~\ref{subsec:onestep_predictive_env}.
\begin{figure}[t]
    \centering
\begin{subfigure}[h]{0.45\linewidth}
    \centering
    \includegraphics[width=\linewidth]{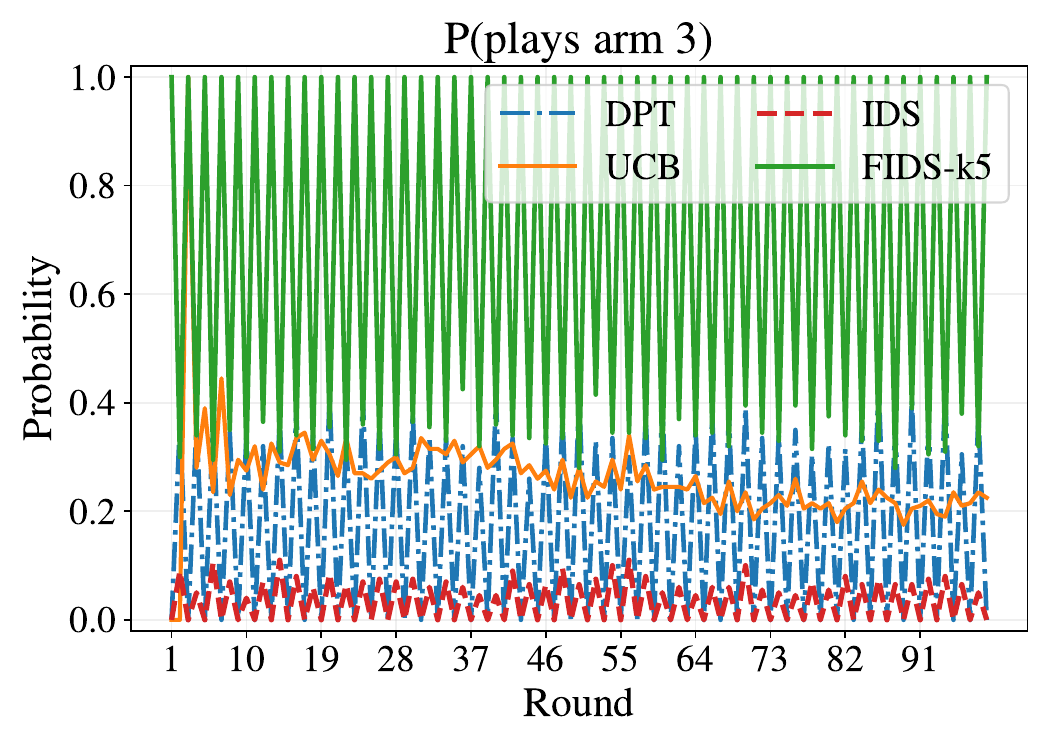}
    \caption{The probability of choosing arm 3}
    \label{fig:onepredict mag}
\end{subfigure}
\hfill
\begin{subfigure}[h]{0.45\linewidth}
    \centering
    \includegraphics[width=\linewidth]{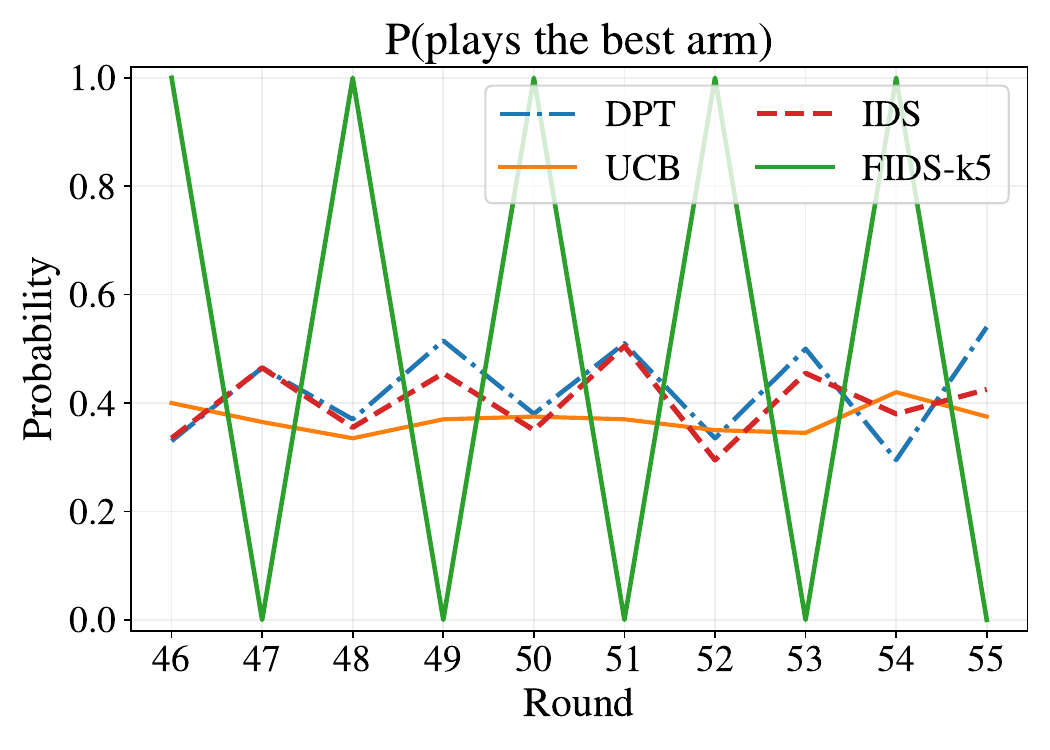}
    \caption{The probability of choosing the true best arm}
    \label{fig:onepredict seg best}
\end{subfigure}
\end{figure}

\section{Other Results}\label{sec:exp details}

\subsection{Proof of Lemma \ref{lemma:2_supp}}\label{app:proof_lemma_2supp}
\begin{proof}
    Denote \(E_{t,a} \coloneqq \mathbb{E}_t\left[R_{t,a}\right]\), \(g_{t,a} \coloneqq I_t\left(\left[A^\star_{t+1},\dots,A^\star_n\right];R_{t,a}\right)\), and \(E_t \coloneqq \left[E_{t,1},\dots,E_{t,K}\right]^T\), \(g_t \coloneqq \left[g_{t,1},\dots,g_{t,K}\right]^T\). Our objective function can then be written as \(\rho(\pi_t) \coloneqq \pi_t \cdot E_t + \sqrt{\Gamma \cdot \pi_t\cdot g_t}\), where \(\Gamma = 2\sigma^2K\) Consider some fixed maximizer \(\pi^\star\) of Equation~\eqref{eq:equivalent definition of FIDS}. The partial derivative of \(\rho\) with respect to \(\pi_{t,a}\) at \(\pi^\star\) is \(\frac{\partial \rho}{\partial \pi_{t,a}} \left(\pi^\star\right) \coloneqq E_{t,a} + \frac{\Gamma g_{t,a}}{2\sqrt{\Gamma g_{t} \cdot \pi^\star}}\).

    Denote \(d^\star \coloneqq \max_{a \in \mathcal{A}} \frac{\partial \rho}{\partial \pi_{t,a}}\left(\pi^\star\right)\). The important observation here is that for any \(a\) such that \(\pi^\star_a > 0\), \(\frac{\partial \rho}{\partial \pi_{t,a}}\left(\pi^\star\right) = d^\star\), because otherwise one must be able to construct a new valid policy by transferring some probability from \(a\) to other arms with larger partial derivative and that will only increase the objective value. Therefore, for all \(a \in B_t \coloneqq \left\{a \in \mathcal{A} : \pi^\star_a > 0\right\}\),
    \begin{align}\label{eq:test}
        E_{t,a} + \frac{\Gamma g_{t,a}}{2\sqrt{\Gamma g_{t} \cdot \pi^\star}} = d^\star
    \end{align}

    Reorder the set \(B_t\) such that \(E_{t,a_1} \geq \dots\geq E_{t,a_{\left|B_t\right|}}\). Note that there always exists a \(\beta \in [0,1]\) such that \(\sum_{a \in B_t} \pi^\star_a E_{t,a} = \beta E_{t,a_1} + (1-\beta) E_{t,a_{\left|B_t\right|}}\), which, by Equation~\eqref{eq:test}, also implies that \(\sum_{a \in B_t} \pi^\star_a g_{t,a} = \beta g_{t,a_1} + (1-\beta) g_{t,a_{\left|B_t\right|}}\). Therefore, we can construct a new policy that only assigns possibilities on the two arms \(a_1\) and \(a_{\left|B_t\right|}\), and this policy has the same objective value as \(\pi^\star\) and thus optimal.
\end{proof}

\subsection{Analytic Solutions}\label{sec:analytic sol}
Depending on which round we are currently in, the mutual information term \(I_t\left(\left[A^\star_{t+1},\dots,A^\star_n\right];R_{t,a}\right)\) is equal to \(I_t\left(\left[Z_1,Z_2\right];R_{t,a}\right)\), if we are in the first half, or \(I_t\left(Z_2;R_{t,a}\right)\) if we are in the second half. By independency, \(I_t\left(\left[Z_1,Z_2\right];R_{t,a}\right) = I_t\left(Z_1;R_{t,a}\right) + I_t\left(Z_2;R_{t,a}\right)\).

In the first half, by independency, when \(a \neq 6\) \(I_t\left(\left[Z_1,Z_2\right];R_{t,a}\right) = I_t\left(Z_1;R_{t,a}\right)\), while \(a = 6\) \(I_t\left(\left[Z_1,Z_2\right];R_{t,a}\right) = I_t\left(Z_2;R_{t,6}\right)\). The main problem is to calculate the posterior distribution. Essentially, we need to maintain \(q_t(Z_1)\) and \(w_t(Z_2)\).

if \(a_t = 6\), \(q_{t+1}(Z_1)=q_t(Z_1)\), and the posterior update rule for \(w_{t+1}(Z_2)\) is
\begin{align*}
    w_{t+1}(Z_2) \propto w_t(Z_2)\mathcal{N}\left(r_t;\frac{Z_2}{7},0.1^2\right)
\end{align*}

While for \(a_t \neq 6\), \(w_{t+1}(Z_2)=w_t(Z_2)\), and the posterior update rule for \(q_{t+1}(Z_1)\) is
\begin{align*}
    q_{t+1}\left(Z_1\right) \propto q_t(Z_1) \mathcal{N}\left(r_t; \mu_z,0.1^2\right)
\end{align*}

This equation can be calculated analytically because \(z\) is finite.
The posterior reward distribution can also be calculated by
\begin{align*}
    \mathbb{P}_t\left(R_{t,a}=r\right) = \sum_z \mathbb{P}_t\left(R_{t,a}=r | Z_1 = z\right)\mathbb{P}_t\left(Z_1=z\right)
\end{align*}

In the second half, the posterior update rule for \(w_{t+1}(Z_2)\) is
\begin{align*}
    w_{t+1}(Z_2) \propto w_t(Z_2)\mathcal{N}\left(r_t;\mu_z,0.5^2\right)
\end{align*}

\subsection{Other Lemmas}
\begin{proposition}\label{prop:iid examples}
    In Example \ref{ex:i.i.d. examples}, FIDS is the optimal policy and collects \(0.35\cdot n\) cumulative reward, while TS and IDS collect \(0.325\cdot n\) and \(0.3294 \cdot n\) cumulative reward respectively.
\end{proposition}

\begin{proof}
    Notice that \(\mathbb{E}_t\left[R_{t,1}\right] = 0.35, \mathbb{E}_t\left[R_{t,2}\right]=0.3\) for every \(t\). Thus, the optimal policy is to choose arm 1 w.p. 1 at every \(t\). FIDS will actually do that by noticing that \(I_t\left(\left[A_{t+1}^\star,\dots, A_n^\star\right];R  _{t,a}\right)=0\) for every \(t\) and \(a\). And this will give cumulative reward \(0.35 \cdot n\).

    TS will play each arm with equal probability at each round, and the cumulative reward is \(0.5 \cdot (0.35+0.3) \cdot n = 0.325\cdot n\).

     We calculate the mutual information terms \(I_t\left(A_t^\star; (R_{t,1})\right)\) and \(I_t\left(A_t^\star; (R_{t,2})\right)\) explicitly as shown below
     
    \(I_t\left(A_t^\star; (R_{t,1})\right) = \mathbb{P}(A_t^\star = 1)\text{KL}(\mathbb{P}(R_{t,1} \in \cdot | A_t^\star=1) || \mathbb{P}(R_{t,1} \in \cdot)) + \mathbb{P}(A_t^\star = 2)\text{KL}(\mathbb{P}(R_{t,1} \in \cdot | A_t^\star=2) || \mathbb{P}(R_{t,1} \in \cdot)) \approx 0.0055\)

     \(I_t\left(A_t^\star; (R_{t,2})\right) = \mathbb{P}(A_t^\star = 1)\text{KL}(\mathbb{P}(R_{t,2} \in \cdot | A_t^\star=1) || \mathbb{P}(R_{t,2} \in \cdot)) + \mathbb{P}(A_t^\star = 2)\text{KL}(\mathbb{P}(R_{t,2} \in \cdot | A_t^\star=2) || \mathbb{P}(R_{t,2} \in \cdot)) \approx 0.0242\)
     
     And then we calculate the IDS policy by solving Equation~\eqref{eq:def of ids} - \(\pi_t^{\text{IDS}}(1) \approx 0.588\) and \(\pi_t^{\text{IDS}}(2) \approx 0.412\). The cumulative reward will be \((0.588 \cdot 0.35 + 0.412 \cdot 0.3)\cdot n = 0.3294\cdot n\). 
\end{proof}

\begin{lemma}\label{lemma:env2 subg}
    Fix constants \(\alpha,\beta, \sigma \in \mathbb{R}\), for the environment with \(\mu_{t,a} \in [\alpha,\beta]\) and \(R_{t,a} \sim \mathcal{N}(\mu_{t,a}, \sigma^2_{t,a})\) where \(\sigma_{t,a}\) is a deterministic non-negative constant no larger than \(\sigma\). Then, conditioned on \(\mathcal{F}_t\),
    \begin{align*}
        R_{t,a} - \mathbb{E}\left[R _{t,a}\mid \mathcal{F}_t\right]
    \end{align*} is \(\sqrt{\left((b-a)/2\right)^2 + \sigma^2}\) sub-Gaussian.
\end{lemma}
\begin{proof}
Note that \(R_{t,a} - \mathbb{E}\left[R _{t,a}\mid \mathcal{F}_t\right] = \mu_{t,a} + \epsilon - \mathbb{E}\left[\mu_{t,a} \mid \mathcal{F}_t\right]\). Therefore, 
    \begin{align*}
        \mathbb{E}\left[\exp\left(\lambda\left(R_{t,a} - \mathbb{E}\left[R _{t,a}\mid \mathcal{F}_t\right]\right)\right) \mid \mathcal{F}_t\right]
        &= \mathbb{E}\left[\exp\left(\lambda\left(\mu_{t,a} + \epsilon - \mathbb{E}\left[\mu_{t,a} \mid \mathcal{F}_t\right]\right)\right)\mid \mathcal{F}_t\right]\\
        &= \mathbb{E}\left[\exp\left(\lambda\left(\mu_{t,a} - \mathbb{E}\left[\mu_{t,a}\mid \mathcal{F}_t\right]\right)\right) \cdot \exp\left(\lambda\epsilon\right) \mid \mathcal{F}_t\right]\\
        &= \mathbb{E}\left[\exp\left(\lambda\left(\mu_{t,a} - \mathbb{E}\left[\mu_{t,a}\mid \mathcal{F}_t\right]\right)\right) \mid \mathcal{F}_t\right] \cdot \mathbb{E}\left[\exp(\lambda \epsilon)\right]\\
        &\leq \exp\left(\lambda^2(\beta-\alpha)^2/8\right) \cdot \exp\left(\lambda^2\sigma^2/2\right)\\
        &= \exp\left(\frac{\lambda^2\left(\left(\beta-\alpha\right)^2/4 + \sigma^2\right)}{2}\right)
    \end{align*}
\end{proof}


\newpage
\clearpage

\end{document}